\documentclass[11pt]{article}
\usepackage[preprint]{acl}
\usepackage{times}
\usepackage{latexsym}
\usepackage[T1]{fontenc}
\usepackage[utf8]{inputenc}
\usepackage{microtype}
\usepackage{comment}
\usepackage{amsmath,amssymb,amsthm}
\usepackage{booktabs,tabularx,graphicx,xcolor}
\usepackage{tikz}
\usepackage{float}
\usepackage{algorithm}
\usepackage{algpseudocode}
\usetikzlibrary{arrows.meta,positioning,patterns}
\newtheorem{proposition}{Proposition}
\newtheorem{theorem}{Theorem}
\usepackage{xurl}
\usepackage[table]{xcolor}
\definecolor{highrisk}{rgb}{0.7,0,0}
\definecolor{ndcolor}{rgb}{0,0.3,0.6}
\definecolor{nfcolor}{rgb}{0.4,0,0.5}
\definecolor{udcolor}{rgb}{0,0.4,0.4}
\newcommand{\hr}[1]{\textcolor{highrisk}{#1}}
\newcommand{\nd}[1]{\textcolor{ndcolor}{#1}}
\newcommand{\nf}[1]{\textcolor{nfcolor}{#1}}
\newcommand{\ud}[1]{\textcolor{udcolor}{#1}}

\newcommand{\E}{\mathbb{E}}
\newcommand{\Prob}{\mathbb{P}}
\newcommand{\ind}{\mathbf{1}}

\title{Calibration and Risk-Budgeted Decisions with System One Models}

\author{Soundouss Messoudi \\ 
  Heudiasyc - UMR CNRS 7253 \\
  Université de Technologie de Compiègne \\
  \texttt{soundouss.messoudi@hds.utc.fr} \\\And
  Youness Mansar \\
  Independent Researcher \\
  France \\
   \\
   \\}

\begin{document}
\maketitle
\begin{abstract}
LLM workflows and agents rely on a set of binary decisions like allow/block an input, or discard/retain a retrieved document. When human deferral is unavailable, uncertain cases must still receive an action. We study System One decision scores in this setting by defaulting to a conservative fallback when uncertain, and controlling the false omission risk within the subset of inputs assigned the utility-maximizing action. Across six datasets covering moderation, safety, and reranking, we compare general System One models with task baselines and a local open-weights fine-tuned System One model. We show that naive baselines for threshold selection exceed the target risk in roughly half of settings with defined test risk, compared with 7.8\% for Learn Then Test (LTT), which leverages finite-sample statistical guarantees. Our results show why relying on a probability-like interface is insufficient for deployment decisions and how risk-budgeted selection trades useful action against observed error.
\end{abstract}

\section{Introduction}

While autoregressive Large Language Models (LLMs) excel at generation and reasoning, their token-by-token generation paradigm induces latency bottlenecks. This makes them ill-suited for fast, structured decisions in settings under tight latency and cost budgets, such as model routing \citep{ong2025routellm}, content moderation \citep{yang2023towards}, prompt-injection detection \citep{guo2026berm}, and safety gating \citep{han2024wildguard}. This mismatch motivates lightweight decision models that expose actionable scores for downstream use.

\begin{figure}[!t]
\centering
\begin{tikzpicture}[x=1cm,y=1cm,font=\footnotesize]
\definecolor{actionzero}{RGB}{0,114,178}
\definecolor{actionone}{RGB}{230,159,0}
\definecolor{defercolor}{RGB}{130,130,130}
\node[anchor=west,font=\bfseries\footnotesize] at (0,4.95) {(a) Selective policy: a middle abstention region};
\begin{scope}[yshift=3.45cm]
  \begin{scope}
    \clip (0,0) rectangle (2.2,1.12);
    \fill[actionzero!55] (0,0) -- plot[domain=0:6.8,samples=120] (\x,{0.10+0.78*exp(-\x/0.55)+1.0*exp(-(6.8-\x)/0.42)}) -- (6.8,0) -- cycle;
  \end{scope}
  \begin{scope}
    \clip (2.2,0) rectangle (4.6,1.12);
    \fill[defercolor!12] (0,0) -- plot[domain=0:6.8,samples=120] (\x,{0.10+0.78*exp(-\x/0.55)+1.0*exp(-(6.8-\x)/0.42)}) -- (6.8,0) -- cycle;
    \fill[pattern=dots,pattern color=black!65] (0,0) -- plot[domain=0:6.8,samples=120] (\x,{0.10+0.78*exp(-\x/0.55)+1.0*exp(-(6.8-\x)/0.42)}) -- (6.8,0) -- cycle;
  \end{scope}
  \begin{scope}
    \clip (4.6,0) rectangle (6.8,1.12);
    \fill[actionone!65] (0,0) -- plot[domain=0:6.8,samples=120] (\x,{0.10+0.78*exp(-\x/0.55)+1.0*exp(-(6.8-\x)/0.42)}) -- (6.8,0) -- cycle;
    \fill[pattern=north east lines,pattern color=black!45] (0,0) -- plot[domain=0:6.8,samples=120] (\x,{0.10+0.78*exp(-\x/0.55)+1.0*exp(-(6.8-\x)/0.42)}) -- (6.8,0) -- cycle;
  \end{scope}
  \draw[gray!80,thick] plot[domain=0:6.8,samples=120] (\x,{0.10+0.78*exp(-\x/0.55)+1.0*exp(-(6.8-\x)/0.42)});
  \draw[gray] (0,0)--(6.8,0);
  \draw[thin] (2.2,0)--(2.2,.92) node[above] {$\tau_{\mathrm{low}}$};
  \draw[thin] (4.6,0)--(4.6,.92) node[above] {$\tau_{\mathrm{high}}$};
  \fill[actionzero!22] (0,-.4) rectangle (2.2,-.08);
  \fill[defercolor!25] (2.2,-.4) rectangle (4.6,-.08);
  \fill[actionone!30] (4.6,-.4) rectangle (6.8,-.08);
  \node at (1.1,-.24) {Action 0};
  \node at (3.4,-.24) {Abstain / defer};
  \node at (5.7,-.24) {Action 1};
  \foreach \x/\lab in {0/0,3.4/0.5,6.8/1} {\node[below] at (\x,-.42) {$\lab$};}
\end{scope}
\node[anchor=west,font=\bfseries\footnotesize] at (0,2.25) {(b) Our binary policy: no abstention action};
\begin{scope}[yshift=.65cm]
  \begin{scope}
    \clip (0,0) rectangle (2.2,1.12);
    \fill[actionzero!55] (0,0) -- plot[domain=0:6.8,samples=120] (\x,{0.10+0.78*exp(-\x/0.55)+1.0*exp(-(6.8-\x)/0.42)}) -- (6.8,0) -- cycle;
  \end{scope}
  \begin{scope}
    \clip (2.2,0) rectangle (6.8,1.12);
    \fill[actionone!65] (0,0) -- plot[domain=0:6.8,samples=120] (\x,{0.10+0.78*exp(-\x/0.55)+1.0*exp(-(6.8-\x)/0.42)}) -- (6.8,0) -- cycle;
    \fill[pattern=north east lines,pattern color=black!45] (0,0) -- plot[domain=0:6.8,samples=120] (\x,{0.10+0.78*exp(-\x/0.55)+1.0*exp(-(6.8-\x)/0.42)}) -- (6.8,0) -- cycle;
  \end{scope}
  \draw[gray!80,thick] plot[domain=0:6.8,samples=120] (\x,{0.10+0.78*exp(-\x/0.55)+1.0*exp(-(6.8-\x)/0.42)});
  \draw[gray] (0,0)--(6.8,0);
  \draw[thin] (2.2,0)--(2.2,.92) node[above] {$\widehat\tau_{\mathrm{LTT}}$};
  \fill[actionzero!22] (0,-.53) rectangle (2.2,-.08);
  \fill[actionone!30] (2.2,-.53) rectangle (6.8,-.08);
  \node[align=center] at (1.1,-.30) {Allow / discard};
  \node[align=center] at (4.5,-.30) {Block / retain (fallback)};
  \foreach \x/\lab in {0/0,3.4/0.5,6.8/1} {\node[below] at (\x,-.54) {$\lab$};}
\end{scope}
\node at (3.4,-.45) {Positive-event score $s(x)$};
\end{tikzpicture}
\caption{Binary actions' illustration. Under calibrated probabilities, scores near $0.5$ indicate uncertain labels: a selective policy can abstain. Our policy instead assigns every input an action via an LTT-selected threshold.}
\label{fig:motivation}
\end{figure}

\emph{System One models}, first introduced with Jev \citep{jev2026}, address this by mapping inputs to bounded decision scores rather than generating unrestricted text. Crucially, these models claim to output well-calibrated probabilities \citep{jev2026}. This suggests an intuitive policy: under exact calibration, thresholding at $\alpha$ guarantees a False Omission Rate (the risk among inputs scored below $\alpha$) bounded by $\alpha$. However, optimizing System One models for calibration during training cannot guarantee test-time calibration across varied downstream distributions \citep{guo2017calibration}. 

The deployment challenge is therefore twofold: determining which operational action the system should take when a model's prediction is uncertain, and guaranteeing safety when its scores are uncalibrated. Moreover, these decisions have severe asymmetric effects that change what uncertainty should trigger (Figure~\ref{fig:motivation}). Selective classification typically mitigates risk through abstention, deferring difficult examples to a human or another model \citep{geifman2017selective}. We focus on fully automated environments where abstention is not a viable action. In these forced-action settings, the conservative fallback is to block an input (for safety) or retain a document (for filtering).

Always defaulting to this fallback avoids the targeted false-negative error but degrades system utility through severe over-refusal or context bloat. The practical objective is therefore to depart from the fallback action as often as possible, while controlling the errors within that newly selected subset. We express this constraint as the False Omission Rate (FOR), $\Prob(Y=1\mid\widehat Y=0)$, where $\widehat Y=0$ denotes the utility-maximizing choice (allow an input or discard a document). We apply Learn Then Test (LTT) \citep{angelopoulos2025learn} to statistically certify threshold regions that satisfy the FOR constraint, effectively substituting selective abstention with the conservative fallback.

We study three linked questions: (Q1) Do the supplied scores reliably represent positive-event probabilities and sufficiently separate the classes? (Q2) How do Naive, Empirical, and LTT threshold selection policies trade observed excess risk against utility? (Q3) How do these outcomes vary across System One models, specialist baselines, and a local open-weights fine-tuned System One model? 

Evaluating on six datasets, we show that safely deploying automated decisions under asymmetric risk needs task-specific validation and rigorous risk-budgeting, rather than relying on the assumed universal calibration of any System One model.

\section{Related Work}
\paragraph{Selective Prediction and Risk Control.}
Selective classification trades coverage for accuracy \citep{geifman2017selective}. An abstention option delegates uncertain inputs. Instead, our setting fixes two operational actions and measures error within the negative-action subset. Risk-controlling prediction sets and Learn Then Test (LTT) \citep{bates2021distribution,angelopoulos2025learn} use held-out data to support statistical risk. LTT is not restricted to abstention; we apply its conditional-risk formulation to binary action policies.

\paragraph{Probability Calibration.}
Because neural network calibration metrics are notoriously sensitive to binning \citep{guo2017calibration,nixon2019measuring}, proper scores like Brier are essential for broader probability assessment \citep{ferrer2025evaluating}. In LLMs, calibration has been examined for question answering \citep{jiang2021can}, self-evaluation \citep{kadavath2022language}, and elicited confidence \citep{tian2023just,xiong2024can}. These studies motivate examining how scores are obtained and what their probabilities actually mean.

\paragraph{Typed-decision evaluation.}
Datasets such as WildGuard \citep{han2024wildguard} and ATBench \citep{li2026atbench} provide foundations for evaluating moderation and trajectory safety. Recent studies of typed-decision models explore aggregate calibration and error ranking \citep{hu2026safety}, as well as probability coherence across logically related questions \citep{li2026beyond}. However, they do not certify a selected action policy's error rate. We distinguish our work by connecting System One scores to risk control across moderation, safety, and reranking.

\section{Problem Formulation}
\label{sec:background}

\subsection{Actions and Asymmetric Risk}
Let $(X,Y)\sim P$ denote an input and its binary ground-truth label drawn from the evaluation population $P$, with $Y\in\{0,1\}$. In our high-stakes language-model workflow, we have:
\begin{itemize}
    \item $Y=1$ represents the \textbf{high-stakes target} (e.g., an unsafe prompt injection attempt, a harmful message, or a highly relevant document).
    \item $Y=0$ represents the \textbf{benign background} (e.g., a safe prompt, an allowable message, or an irrelevant document).
\end{itemize}

A frozen scorer returns $s(x)\in[0,1]$, oriented so that larger scores express stronger evidence for $Y=1$. We use threshold $\tau\in[0,1]$ to define the binary decision rule $\widehat Y=\pi_\tau(x)$:
\begin{equation*}
\pi_\tau(x) = 
\begin{cases} 
0 & \text{if } s(x) < \tau, \\
1 & \text{if } s(x) \ge \tau.
\end{cases}
\end{equation*}
Here, Action 0 is the utility-maximizing choice (allowing safe inputs or discarding irrelevant documents) taken under high confidence. Action 1 serves as the conservative fallback (blocking or retaining), triggered by either high confidence in class 1 or predictive uncertainty.

In this setting, the costs of errors are asymmetric:
\begin{itemize}
    \item \textbf{False Negative (Predict 0 / True 1):} This is a \textbf{critical failure}. It means allowing a toxic interaction to execute or permanently discarding a crucial piece of evidence.
    \item \textbf{False Positive (Predict 1 / True 0):} This is a \textbf{utility loss}. It causes over-refusal of benign inputs or over-inclusion that inflates an LLM context with irrelevant documents. It is treated as the lower-priority error in our setting.
\end{itemize}

Because predicting 0 when the true label is 1 represents the critical system failure, the primary risk $R(\tau)$ is defined on the negative-action subset of data (i.e., when the system predicts 0). This is the False Omission Rate (FOR):
\begin{equation}
\begin{split}
R(\tau) &= \Prob(Y=1\mid \widehat Y=0) \\
        &= \Prob(Y=1\mid s(X)<\tau).
\end{split}
\label{eq:risk}
\end{equation}

\subsection{The Optimization Trade-Off}\label{sec:tradeoff}
While the high-stakes constraint requires defaulting to Action 1 when uncertain, always predicting 1 gives a practically useless system with an empty negative-action subset, where its measure is $\mu(\tau) = \Prob(s(X)<\tau) = 0$. 

The system's primary optimization goal is to maximize the size of this negative-action subset subject to a strict population-level risk budget $\alpha \in (0,1)$. Ideally, we seek a threshold $\tau$ that solves:
\begin{equation*}
\begin{aligned}
\max_{\tau} \quad & \mu(\tau) = \Prob(s(X)<\tau)\\
\text{subject to} \quad & R(\tau) \le \alpha
\end{aligned}
\end{equation*}

Maximizing this set is linked to minimizing the number of false positives, $\sum_{i=1}^n \ind\{\widehat Y_{\tau,i} = 1, Y_i = 0\}$. As $\mu(\tau)$ grows, fewer true zeros are incorrectly forced into the fallback action (Action 1), driving down the incidence of over-refusal and context bloat.

However, because the true data distribution is unknown, we cannot directly access the population risk $R(\tau)$. The objective is therefore to establish an empirical threshold selection procedure that selects the largest certified candidate $\tau$ to minimize utility loss, while relying on statistical testing to rigorously ensure the critical risk of catastrophic FOR is controlled at budget $\alpha$.

\section{Threshold Selection and Risk Control}
\label{sec:method}

\begin{algorithm}[t]
\caption{Fixed-order threshold certification.}
\label{alg:fixed_order_certification}
\begin{algorithmic}[1]
\Require Frozen scorer; independently fixed ordered thresholds $(\tau_1,\ldots,\tau_M)$; independent $D_{\mathrm{cal}}$; $(\alpha,\delta)$
\State Initialize the certified threshold set $\widehat{\mathcal T} \leftarrow \varnothing$
\For{$j = 1, \dots, M$}
    \State Compute $(N_j, K_j, p_j)$
    \If{$p_j > \delta$}
        \State \textbf{break} \Comment{Stop the sequence}
    \Else
        \State $\widehat{\mathcal T}\leftarrow \widehat{\mathcal T}\cup\{\tau_j\}$
    \EndIf
\EndFor
\If{$\widehat{\mathcal T} = \varnothing$}
    \State \Return $\bot$ \Comment{Non-deployment (Fallback to Action 1)}
\EndIf
\State \Return $\widehat\tau \in \arg\max_{\tau_j \in \widehat{\mathcal T}} \widehat \mu_{\mathrm{cal},j}$ \Comment{Maximize negative-action subset measure}
\Ensure Selected $\widehat\tau$, or non-deployment ($\bot$)
\end{algorithmic}
\end{algorithm}

\subsection{Certifying the False Omission Rate}
\label{sec:ltt}
To solve the finite-sample optimization trade-off introduced in Section~\ref{sec:tradeoff}, we must rely on statistical testing. Given the target risk budget $\alpha$ and a statistical error level $\delta \in (0,1)$, the formal objective is to find a threshold $\tau$ that maximizes the measure of the negative-action subset while guaranteeing the risk bound holds with high probability:
\begin{equation*}
\begin{aligned}
\max_{\tau} \quad & \mu(\tau) = \Prob(s(X)<\tau)\\
\text{subject to} \quad & \Prob(R(\tau) \le \alpha) \ge 1 - \delta
\end{aligned}
\end{equation*}

To obtain this guarantee, we specialize Learn Then Test (LTT) \citep{angelopoulos2025learn} to the binary conditional risk in Equation~\ref{eq:risk}. 

Let $\mathcal T=\{\tau_1,\ldots,\tau_M\}$ contain $M$ candidate thresholds, tested in the specified order $(\tau_1,\ldots,\tau_M)$ fixed independently of the calibration\footnote{The word "calibration" in LTT refers to statistical selection, not probability recalibration.} sample. Let $\alpha \in (0,1)$ be the risk budget and $\delta\in(0,1)$ be the error level that bounds the probability of certifying any excessive-risk candidate in this family. Let $D_{\mathrm{fit}}$ denote data used to fit the scorer $s(X)$ and construct the threshold family's grid, and let $D_{\mathrm{cal}}=\{(X_i,Y_i)\}_{i=1}^{n_{\mathrm{cal}}}$ be the separate calibration sample of size $n_{\mathrm{cal}}$. Define:
\begin{align*}
    N_j &=\sum_{i=1}^{n_{\mathrm{cal}}} \ind\{s(X_i)<\tau_j\},\\
    K_j &=\sum_{i=1}^{n_{\mathrm{cal}}} Y_i\ind\{s(X_i)<\tau_j\}.
\end{align*}
For candidate $j$, $N_j$ counts calibration examples assigned Action 0, and $K_j$ counts positive labels among them. For $N_j>0$,
\[
    \widehat R_{\mathrm{cal},j}
    =\frac{K_j}{N_j},
    \qquad
    \widehat \mu_{\mathrm{cal},j}
    =\frac{N_j}{n_{\mathrm{cal}}}.
\]

For a fixed policy under independent and identically distributed (i.i.d.) calibration sampling, conditioning on $N_j=n$ gives
\[
    K_j\mid N_j=n
    \sim\operatorname{Binomial}\!\left(n,R(\tau_j)\right).
\]
To test the null hypothesis $H_j:R(\tau_j)>\alpha$ (i.e., that the policy exceeds the risk budget), we use the lower-tail binomial $p$-value:
\begin{equation*}
    p_j= \sum_{k=0}^{K_j} \binom{N_j}{k} \alpha^k(1-\alpha)^{N_j-k},
\end{equation*}
with $p_j=1$ when $N_j=0$. A small $p_j$ provides the statistical evidence needed for certification. Algorithm~\ref{alg:fixed_order_certification} states the fixed-family procedure to which the theoretical LTT guarantees apply.

\begin{theorem}[Fixed-family LTT specialization \citep{angelopoulos2025learn}]
\label{thm:ltt}
Condition on $D_{\mathrm{fit}}$, which fixes the scorer and ordered candidate family. Suppose that $D_{\mathrm{cal}}$ consists of i.i.d.\ draws from $P$, independently of $D_{\mathrm{fit}}$. Algorithm~\ref{alg:fixed_order_certification} satisfies:
\begin{equation*}
    \Prob_{D_{\mathrm{cal}}} \!\left( \exists\tau\in\widehat{\mathcal T}: R(\tau)>\alpha \right) \leq\delta.
    \label{eq:fwer}
\end{equation*}
Consequently, any rule returning an element of $\widehat{\mathcal T}$ or $\bot$ (non-deployment) satisfies
\begin{equation*}
    \Prob_{D_{\mathrm{cal}}} \!\left( \widehat\tau\neq\bot,\, R(\widehat\tau)>\alpha \right) \leq\delta.
    \label{eq:guarantee}
\end{equation*}
\end{theorem}

\begin{proof}
This theorem is a direct application of fixed-sequence testing (specifically Theorem 2.1) established by \citet{angelopoulos2025learn}, applied to our binary thresholding rule.
\end{proof}

At deployment, the selected $\widehat\tau$ assigns every input to Action 0 or 1. Returning $\bot$ means threshold selection did not establish a certificate for a useful negative-action subset, and the system must default to the fallback action (Action 1). This distinction separates \emph{failure to select a certified policy} from \emph{deferring an individual input}.

\subsection{Why Naive Thresholding Is Not Enough}
\label{sec:bridge}
Let us consider the relationship between population risk and score calibration. Let the random variable $S=s(X)$ be the score and $m(u)=\Prob(Y=1\mid S=u)=\E[Y\mid S=u]$ its calibration function. Exact calibration means $m(S)=S$ except on a probability-zero set. The following identity separates the average model score from the average probability error within the negative-action region.

\begin{proposition}[Calibration--policy identity]\label{prop:bridge}
For $\mu(\tau)>0$,
\begin{equation*}
R(\tau)=\E[S\mid S<\tau]+\E[m(S)-S\mid S<\tau].
\end{equation*}
Consequently,
\begin{equation*}
R(\tau)\leq\tau+\frac{\E[|m(S)-S|\ind\{S<\tau\}]}{\mu(\tau)}.
\end{equation*}
Under exact calibration, $R(\tau)\leq\tau$.
\end{proposition}
The derivation appears in Appendix~\ref{app:proofs}. This identity demonstrates that if a model is perfectly calibrated, naive thresholding achieves bounded risk (i.e., setting $\tau = \alpha$ gives $R(\tau) \le \alpha$). However, it also proves that when scores are uncalibrated, depending on the direction of the error, naive thresholding can allow excessive risk. LTT instead tests the population-risk constraint using finite-sample data, without requiring calibrated model scores. Its guarantee is the calibration-sampling statement in Theorem~\ref{thm:ltt}, not a promise that every finite test estimate is below budget.
\section{Experimental Setup}
\label{sec:experimental_setup}

\begin{table*}[t] 
\centering
\small
\setlength{\tabcolsep}{6pt}
\begin{tabular}{llrrrc}
\toprule
Task & Dataset & Fit & Cal. & Test & Y=1 (\%) \\
\midrule
Moderation & WildGuardMix \citep{han2024wildguard} & 2000 & 849 & 850 & 44.4 \\
Moderation & Aegis \citep{ghosh2024aegis} & 2000 & 579 & 579 & 67.0 \\
Safety & S-Labs \citep{prompt2026} & 2000 & 1051 & 1050 & 50.0 \\
Safety & ATBench \citep{li2026atbench} & 500 & 250 & 250 & 47.6 \\
Reranking & SciDocs \citep{cohan2020specter} & 2000 & 1000 & 1000 & 50.8 \\
Reranking & StackOverflow \citep{muennighoff2023mteb} & 2000 & 1000 & 1000 & 50.1 \\
\bottomrule
\end{tabular}
\caption{Dataset split details. The last column gives the test positive-event prevalence (base rate).}
\label{tab:splits}
\end{table*}

\subsection{Evaluated Models}
We categorize evaluated models into three groups:
\begin{itemize}
    \item \textbf{Baselines:} For moderation and safety, we evaluate Qwen3Guard-Gen-0.6B \citep{zhao2025qwen3guard} as a specialist baseline. For reranking, we evaluate Voyage rerank-2.5-lite \citep{voyage2026}. We also include PCA-Ridge, an L2-regularized logistic regression baseline fitted per dataset on frozen Qwen3-8B embeddings.
    \item \textbf{System One Models:} We evaluate models specifically designed with probability-valued interfaces, including Jev \citep{jev2026}, GPT-6-Luna (via its typed-decisions endpoint) \citep{gpt6luna2026}, Span-01 \citep{span012026}, and Solar-Decide \citep{solardecide2026}.
    \item \textbf{Base vs.\ Fine-Tuned:} To illustrate adaptation dynamics within a single System One model family, we evaluate Kev-0.8B \citep{kev2026} both as released (Kev-base) and after fine-tuning on $\approx10k$ examples for 37 minutes on an RTX 5090 (Kev-ft). Kev-ft is trained on a mixture of all dataset fit splits to improve generalization, limit overfitting, and save compute cost of six separate neural trainings.
\end{itemize}
Other details are provided in Appendix~\ref{app:models}.

\subsection{Tasks and Datasets}
We evaluate our framework across three workflows with six datasets. \textbf{Moderation} uses WildGuardMix and Aegis prompt labels. \textbf{Safety} covers prompt injection on S-Labs and unsafe recorded trajectories on ATBench. \textbf{Reranking} uses SciDocs and StackOverflow query--document pairs to evaluate reranking as a binary filtering decision. 

For each dataset, we use a single seed, label-stratified fit/calibration/test splits (Table~\ref{tab:splits}). Because most evaluated scorers are general pre-trained models, $D_{\mathrm{fit}}$ is only used to fit the PCA-Ridge baseline, fine-tune the Kev-ft model, and construct candidate threshold grids. $D_{\mathrm{cal}}$ is used for threshold selection, and test data evaluates the final fixed policies. Other details are provided in Appendix~\ref{app:data} and Appendix~\ref{app:data2}.

\subsection{Threshold Selection Policies}
\label{sec:candidates}
We evaluate three threshold policies across risk budgets $\alpha\in\{0.05,0.10,0.15,0.20,0.40,0.50\}$:
\begin{enumerate}
    \item \textbf{Naive:} Sets $\tau=\alpha$. As discussed in Section~\ref{sec:bridge}, this assumes exact calibration and includes no finite-sample margin.
    \item \textbf{Empirical:} Selects $\tau$ that yields the lowest False Positive Rate, provided that the observed calibration FOR is at most $\alpha$.
    \item \textbf{LTT:} Applies Algorithm~\ref{alg:fixed_order_certification} with a statistical error level of $\delta=0.10$ to return the largest threshold in the certified prefix. Details on implementation are in Appendix~\ref{app:lttimp}.
\end{enumerate}

\subsection{Evaluation Metrics}
We evaluate the systems on two axes: underlying score quality and operational policy outcomes.

\paragraph{Score Diagnostics.} We assess raw predictions before thresholding with Area Under the Receiver Operating Characteristic (AUROC) and negative-action Average Precision (AP$_0$) for class separation. We report 10-bin expected calibration error (ECE) and Brier score to measure probability alignment.

\paragraph{Operational Utility and Risk.} At deployment, we report the True Negative Rate (TNR) as our primary measure of system utility (e.g., the rate of benign acceptance or irrelevant document removal). Empirical risk is measured by test-set FOR, $\widehat R_{\mathrm{test}}$. To quantify safety violations, we report the positive excess risk, defined as $e=\max(\widehat R_{\mathrm{test}}-\alpha,0)$. 

To ensure statistical rigor, 95\% bootstrap confidence intervals for score diagnostics and primary-budget policy risk/utility appear in Appendix~\ref{app:full-results}. The resampling recipe is in Appendix~\ref{app:metric-details}.

\paragraph{Availability and Resource Use.} We distinguish no LTT certificate (\nd{ND}), no feasible Empirical policy (\nf{NF}), and a returned policy with an empty negative-action test set (\ud{UD}, undefined FOR). Exceedance fractions use defined test risks; status fractions use settings with defined risk. Mean excess includes zero excess for non-exceedances and is expressed in percentage points. We also report inference cost (USD per 1,000 workload examples) and mean recorded request time (ms/example). Full metric definitions and accounting conventions appear in Appendix~\ref{app:metric-details} and Appendix~\ref{app:efficiency}.
\section{Results}
\label{sec:results}
We answer our three research questions by evaluating the held-out test predictions and saved policy selections across all six datasets. Table~\ref{tab:quality} summarizes the underlying score diagnostics and calibration metrics for each model. Table~\ref{tab:outcomes_main} aggregates the frequency and magnitude of risk exceedances alongside overall policy availability. Table~\ref{tab:policies} details the resulting operational utility and empirical risk at the target budget $\alpha=0.15$. Figure~\ref{fig:ltt-operating} visualizes the LTT-selected operating points across varying risk tolerances to demonstrate model differences on WildGuardMix. Finally, Table~\ref{tab:efficiency-main} reports the inference cost and recorded latency to contextualize operational deployment.

\begin{table*}[!t]
\centering
\footnotesize
\setlength{\tabcolsep}{2pt}
\begin{tabular*}{\textwidth}{@{\extracolsep{\fill}}lcccccccccccc@{}}
\toprule
 & \multicolumn{4}{c}{\textbf{Moderation: WildGuardMix}} & \multicolumn{4}{c}{\textbf{Moderation: Aegis}} & \multicolumn{4}{c}{\textbf{Safety: S-Labs}} \\
\cmidrule(lr){2-5}\cmidrule(lr){6-9}\cmidrule(lr){10-13}
Model & ECE $\downarrow$ & Brier $\downarrow$ & AUC $\uparrow$ & AP$_0\uparrow$ & ECE $\downarrow$ & Brier $\downarrow$ & AUC $\uparrow$ & AP$_0\uparrow$ & ECE $\downarrow$ & Brier $\downarrow$ & AUC $\uparrow$ & AP$_0\uparrow$ \\
\midrule
Qwen3Guard & 0.085 & 0.090 & 0.950 & 0.947 & 0.277 & 0.239 & 0.929 & 0.866 & 0.474 & 0.464 & 0.968 & 0.968 \\
PCA-Ridge$^{*}$ & 0.042 & 0.091 & 0.946 & 0.956 & 0.037 & 0.096 & 0.934 & 0.892 & 0.087 & 0.041 & 0.996 & 0.996 \\
Jev & 0.034 & 0.073 & 0.963 & 0.969 & 0.183 & 0.151 & 0.936 & 0.860 & 0.185 & 0.158 & 0.947 & 0.927 \\
GPT-6-Luna & 0.103 & 0.108 & 0.925 & 0.907 & 0.297 & 0.277 & 0.864 & 0.665 & 0.262 & 0.228 & 0.888 & 0.841 \\
Span-01 & 0.189 & 0.238 & 0.765 & 0.746 & 0.500 & 0.472 & 0.565 & 0.332 & 0.441 & 0.437 & 0.288 & 0.368 \\
Solar-Decide & 0.104 & 0.130 & 0.893 & 0.855 & 0.356 & 0.339 & 0.865 & 0.658 & 0.337 & 0.326 & 0.840 & 0.737 \\
Kev-base & 0.121 & 0.186 & 0.835 & 0.863 & 0.257 & 0.232 & 0.834 & 0.704 & 0.225 & 0.203 & 0.944 & 0.940 \\
Kev-ft$^{*}$ & 0.119 & 0.123 & 0.934 & 0.938 & 0.109 & 0.126 & 0.919 & 0.849 & 0.024 & 0.028 & 0.995 & 0.995 \\
\midrule
 & \multicolumn{4}{c}{\textbf{Safety: ATBench}} & \multicolumn{4}{c}{\textbf{Reranking: SciDocs}} & \multicolumn{4}{c}{\textbf{Reranking: StackOverflow}} \\
\cmidrule(lr){2-5}\cmidrule(lr){6-9}\cmidrule(lr){10-13}
Model & ECE $\downarrow$ & Brier $\downarrow$ & AUC $\uparrow$ & AP$_0\uparrow$ & ECE $\downarrow$ & Brier $\downarrow$ & AUC $\uparrow$ & AP$_0\uparrow$ & ECE $\downarrow$ & Brier $\downarrow$ & AUC $\uparrow$ & AP$_0\uparrow$ \\
\midrule
Reranker-Lite & -- & -- & -- & -- & 0.183 & 0.214 & 0.880 & 0.861 & 0.137 & 0.200 & 0.829 & 0.826 \\
PCA-Ridge$^{*}$ & 0.151 & 0.165 & 0.882 & 0.902 & 0.065 & 0.127 & 0.904 & 0.885 & 0.032 & 0.204 & 0.750 & 0.748 \\
Jev & 0.076 & 0.200 & 0.768 & 0.717 & 0.290 & 0.236 & 0.917 & 0.895 & 0.099 & 0.196 & 0.790 & 0.800 \\
GPT-6-Luna & 0.230 & 0.232 & 0.842 & 0.844 & 0.340 & 0.289 & 0.865 & 0.787 & 0.120 & 0.191 & 0.815 & 0.819 \\
Span-01 & 0.101 & 0.196 & 0.779 & 0.744 & 0.175 & 0.189 & 0.864 & 0.834 & 0.224 & 0.302 & 0.607 & 0.576 \\
Solar-Decide & 0.274 & 0.288 & 0.688 & 0.716 & 0.357 & 0.333 & 0.772 & 0.643 & 0.266 & 0.288 & 0.706 & 0.643 \\
Kev-base & 0.255 & 0.303 & 0.457 & 0.538 & 0.120 & 0.234 & 0.734 & 0.724 & 0.051 & 0.245 & 0.599 & 0.559 \\
Kev-ft$^{*}$ & 0.059 & 0.099 & 0.938 & 0.952 & 0.118 & 0.131 & 0.913 & 0.881 & 0.160 & 0.205 & 0.810 & 0.818 \\
\bottomrule
\end{tabular*}
\caption{Test score diagnostics by dataset. AUC is AUROC and AP$_0$ is negative-action average precision. Reranker scores are not advertised probabilities. 95\% intervals in Appendix~\ref{app:full-results}. --: not evaluated. *: supervised adaptation.}\label{tab:quality}
\end{table*}

\subsection{Q1: Calibration is Task-Dependent and Insufficient for Policy Viability}
\label{sec:rq1}
Table~\ref{tab:quality} shows that probability reliability is fundamentally dataset-dependent and no single architecture universally dominates. A bounded interface alone does not guarantee calibration. For example, Jev's ECE degrades from $0.034$ on WildGuardMix to $0.29$ on SciDocs. Universal claims of ``calibrated probabilities'' obscure this severe variability.

Crucially, low global ECE implies neither effective class separation nor policy viability. On StackOverflow, Kev-base achieves a much better ECE than Kev-ft ($0.051$ vs. $0.16$) but exhibits worse discriminative power (AUROC $0.599$ vs. $0.81$). This dynamic illustrates Proposition~\ref{prop:bridge}: while ECE averages absolute error across all inputs, operational risk depends on the conditional error within the selected negative-action set, $\E[m(S)-S \mid S<\tau]$. Thus, aggregate calibration cannot substitute for rigorous, subset-specific policy evaluation.

\subsection{Q2: LTT Prioritizes Safety While Naive Thresholds Frequently Exceed Budgets}
\label{sec:rq2}

Our analysis shows that naive thresholding can be arbitrarily permissive or conservative depending on the dataset. Naive methods routinely trade away safety for utility: on Aegis at $\alpha=0.15$, Span-01 achieves $95.3\%$ utility under Naive thresholding, but its empirical risk catastrophically explodes to $60.9\%$ (Table~\ref{tab:policies}). Conversely, naive thresholds can be unnecessarily restrictive. On WildGuardMix, Jev's naive policy limits utility to $68.7\%$ (at $3.3\%$ risk), whereas LTT safely expands it to $96.8\%$ while remaining within budget ($12.8\%$ risk).

While the Empirical selection rule can sometimes identify high-utility policies (e.g., Kev-ft reaching $99.6\%$ utility at $12.5\%$ risk on S-Labs), it lacks formal guarantees. Table~\ref{tab:outcomes_main} shows that Empirical selection exceeds the target budget in $51.7\%$ of settings (mean excess $1.17\%$). The Naive rule similarly exceeds the budget $42\%$ of the time, concentrating the most extreme failures (mean excess $7.09\%$, with a maximum exceedance of $73.26$ percentage points at $\alpha=0.05$). 

LTT produces substantially fewer and smaller risk exceedances (just $7.8\%$ frequency, mean excess $0.1\%$) because its statistical certification structurally defaults to non-deployment (ND) rather than operate unsafely. This provides a rigorous finite-sample guarantee that, with probability at least $1-\delta$, the system will stay below the risk budget. Consequently, LTT returns fewer policies under strict budgets (establishing certificates for only $21.3\%$ of policies at $\alpha=0.05$), though availability increases as the budget relaxes ($76.6\%$ at $\alpha=0.5$).

\begin{table}[!t]
\centering
\footnotesize
\setlength{\tabcolsep}{4pt}
\begin{tabular}{@{}llccccc@{}}
\toprule
$\alpha$ & Rule & Eval. & Fail & Exceeds & Mean $e$ & Max $e$ \\
\midrule
5 & Naive & 34 & \ud{13 UD} & 23 (67.6\%) & 14.00 & 73.26 \\
     & Emp.  & 23 & \nf{24 NF} & 19 (82.6\%) & 1.73  & 5.81  \\
     & LTT   & 10 & \nd{37 ND} & 0 (0.0\%)   & 0.00  & 0.00  \\
\midrule
10 & Naive & 38 & \ud{9 UD}  & 22 (57.9\%) & 11.22 & 48.87 \\
     & Emp.  & 30 & \nf{17 NF} & 18 (60.0\%) & 1.36  & 6.45  \\
     & LTT   & 20 & \nd{27 ND} & 1 (5.0\%)   & 0.03  & 0.53  \\
\midrule
15 & Naive & 41 & \ud{6 UD}  & 23 (56.1\%) & 9.81  & 45.86 \\
     & Emp.  & 32 & \nf{15 NF} & 20 (62.5\%) & 1.59  & 8.11  \\
     & LTT   & 24 & \nd{23 ND} & 1 (4.2\%)   & 0.07  & 1.60  \\
\midrule
20 & Naive & 43 & \ud{4 UD}  & 20 (46.5\%) & 8.06  & 41.30 \\
     & Emp.  & 34 & \nf{13 NF} & 19 (55.9\%) & 1.49  & 8.57  \\
     & LTT   & 29 & \nd{18 ND} & 3 (10.3\%)  & 0.21  & 3.47  \\
\midrule
40 & Naive & 47 & \ud{0 UD}  & 13 (27.7\%) & 2.03  & 23.39 \\
     & Emp.  & 44 & \nf{3 NF}  & 20 (45.5\%) & 0.91  & 5.70  \\
     & LTT   & 34 & \nd{13 ND} & 4 (11.8\%)  & 0.09  & 1.69  \\
\midrule
50 & Naive & 47 & \ud{0 UD}  & 4 (8.5\%)   & 0.54  & 13.69 \\
     & Emp.  & 46 & \nf{1 NF}  & 12 (26.1\%) & 0.48  & 4.78  \\
     & LTT   & 36 & \nd{11 ND} & 3 (8.3\%)   & 0.10  & 1.79  \\
\midrule
\rowcolor{gray!20} \textbf{$\sum$} & Naive & 250 & \ud{32 UD} & 105 (42.0\%) & 7.09 & 73.26 \\
\rowcolor{gray!20}                & Emp.  & 209 & \nf{73 NF} & 108 (51.7\%) & 1.17 & 8.57  \\
\rowcolor{gray!20}                & LTT   & 153 & \nd{129 ND}& 12 (7.8\%)   & 0.10 & 3.47  \\
\bottomrule
\end{tabular}
\caption{Outcomes by risk budget $\alpha$ (\%) across all 47 model--dataset pairs. Excess risk is $e=\max(\widehat R-\alpha,0)$.}
\label{tab:outcomes_main}
\end{table}

\subsection{Q3: Adaptation Expands Availability, Though General System One Models Remain Highly Competitive}

General System One interfaces offer a practical plug-and-play starting point, though risk-controlled deployment still requires finite-sample data for policy selection. As Figure~\ref{fig:ltt-operating} shows on WildGuardMix ($\alpha=0.15$), Jev and GPT-6-Luna achieve exceptional LTT utility ($96.8\%$ and $92.6\%$, respectively), rivaling the specialist Qwen3Guard baseline ($95.1\%$). Thus, high utility does not strictly require fine-tuning, although some unadapted models (Span-01, Solar-Decide) fail entirely (Table~\ref{tab:policies}).

However, local open-weight model fine-tuning massively improves policy availability over unadapted base models. At $\alpha=0.15$, Kev-base returns an LTT policy on only one of six datasets, whereas Kev-ft succeeds on all six (Table~\ref{tab:policies}). Across all budgets, this adaptation expands availability from 13/36 to 31/36 settings. Even when both return policies, as on WildGuardMix, utility jumps significantly from $41.9\%$ for Kev-base to $82.5\%$ for Kev-ft (Figure~\ref{fig:ltt-operating}).

Fine-tuning does not universally dominate: Jev outperforms Kev-ft on WildGuardMix with a $96.8\%$ utility at this budget, and the PCA-Ridge baseline remains competitive across several tasks. Because no single architecture solves forced-action deployment, a practical strategy is to first evaluate a general System One model and fall back to local open-weights model adaptation if it proves too conservative. Crucially, all approaches demand strict task-specific empirical validation.

\subsection{Local Open-Weights Models Provide the Most Efficient Cost-Latency Balance}
Operational deployment must balance policy usefulness against inference latency and financial cost (Table~\ref{tab:efficiency-main}). Proprietary closed-weight models exhibit widely varying degrees of cost and latency depending on the endpoint. Crucially, simpler baseline architectures do not automatically guarantee faster execution. The PCA-Ridge baseline is exceptionally cheap to evaluate numerically (\$$0.0026$ per 1,000 examples) but it has a severe latency penalty ($5{,}247$ ms/example). This delay comes from the requisite embedding extraction pipeline, which must fully process the input through Qwen3-8B before the lightweight logistic regression can execute. By contrast, an end-to-end local fine-tuned model like Kev-ft provides the most optimal balance, delivering highly competitive safety utility at just \$$0.0030$ and $98$ ms/example, avoiding both the API costs of proprietary models and the sequential embedding bottlenecks of PCA-Ridge.

\begin{figure}[t]
\centering
\includegraphics[width=\columnwidth]{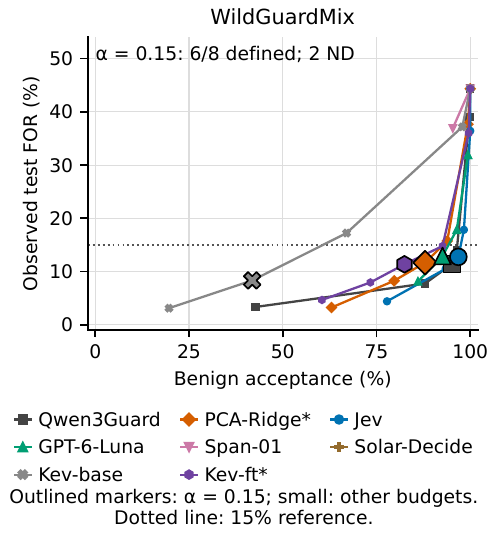}
\caption{LTT-selected risk--utility operating points on WildGuardMix. Lower test FOR and greater benign acceptance are preferable. Missing policies are omitted.}
\label{fig:ltt-operating}
\end{figure}

\begin{table*}[!t]
\centering
\footnotesize
\setlength{\tabcolsep}{2pt}
\begin{tabular*}{\textwidth}{@{\extracolsep{\fill}}lcccccccccccc@{}}
\toprule
 & \multicolumn{6}{c}{\textbf{Moderation: WildGuardMix}} & \multicolumn{6}{c}{\textbf{Moderation: Aegis}} \\
\cmidrule(lr){2-7}\cmidrule(lr){8-13}
 & \multicolumn{2}{c}{Naive} & \multicolumn{2}{c}{Empirical} & \multicolumn{2}{c}{LTT} & \multicolumn{2}{c}{Naive} & \multicolumn{2}{c}{Empirical} & \multicolumn{2}{c}{LTT} \\
Model & Risk & Util & Risk & Util & Risk & Util & Risk & Util & Risk & Util & Risk & Util \\
\midrule
Qwen3Guard & 10.8 & 94.3 & 12.2 & 95.8 & 11.4 & 95.1 & \hr{38.4} & 93.2 & \hr{18.2} & 72.8 & 11.5 & 60.7 \\
PCA-Ridge$^{*}$ & 4.2 & 67.4 & 13.1 & 91.3 & 11.7 & 87.9 & 1.4 & 36.1 & \hr{17.3} & 72.8 & 9.5 & 64.9 \\
Jev & 3.3 & 68.7 & \hr{16.2} & \textbf{98.1} & 12.8 & 96.8 & \hr{20.8} & 83.8 & \hr{23.1} & 85.3 & 11.7 & 67.0 \\
GPT-6-Luna & 12.6 & 92.6 & \hr{15.6} & 96.0 & 12.9 & 92.6 & \hr{43.4} & 94.8 & \nf{NF} & \nf{NF} & \nd{ND} & \nd{ND} \\
Span-01 & \hr{26.3} & 81.2 & \nf{NF} & \nf{NF} & \nd{ND} & \nd{ND} & \hr{60.9} & 95.3 & \nf{NF} & \nf{NF} & \nd{ND} & \nd{ND} \\
Solar-Decide & \hr{18.1} & 93.9 & 13.9 & 84.8 & \nd{ND} & \nd{ND} & \hr{50.7} & \textbf{97.9} & \nf{NF} & \nf{NF} & \nd{ND} & \nd{ND} \\
Kev-base & 0.0 & 4.0 & \hr{15.4} & 62.6 & 8.3 & 41.9 & 12.5 & 7.3 & \hr{18.6} & 29.8 & \nd{ND} & \nd{ND} \\
Kev-ft$^{*}$ & 11.3 & 82.9 & 13.4 & 87.5 & 11.4 & 82.5 & \hr{17.7} & 70.7 & \hr{19.5} & 71.2 & 14.5 & 64.9 \\
\midrule
 & \multicolumn{6}{c}{\textbf{Safety: S-Labs}} & \multicolumn{6}{c}{\textbf{Safety: ATBench}} \\
\cmidrule(lr){2-7}\cmidrule(lr){8-13}
 & \multicolumn{2}{c}{Naive} & \multicolumn{2}{c}{Empirical} & \multicolumn{2}{c}{LTT} & \multicolumn{2}{c}{Naive} & \multicolumn{2}{c}{Empirical} & \multicolumn{2}{c}{LTT} \\
Model & Risk & Util & Risk & Util & Risk & Util & Risk & Util & Risk & Util & Risk & Util \\
\midrule
Qwen3Guard & \hr{47.8} & \textbf{100.0} & 14.7 & 94.1 & 13.5 & 93.1 & -- & -- & -- & -- & -- & -- \\
PCA-Ridge$^{*}$ & 0.4 & 87.0 & 11.6 & 99.8 & 10.0 & 99.8 & 0.0 & 1.5 & 13.6 & 72.5 & 10.1 & 67.9 \\
Jev & 11.2 & 86.1 & 13.9 & 95.4 & 10.2 & 79.0 & \hr{33.3} & 3.1 & \nf{NF} & \nf{NF} & \nd{ND} & \nd{ND} \\
GPT-6-Luna & \hr{27.0} & 94.3 & 13.2 & 71.4 & \nd{ND} & \nd{ND} & 5.6 & 26.0 & \hr{16.3} & 62.6 & 12.0 & 55.7 \\
Span-01 & \hr{47.5} & \textbf{100.0} & \nf{NF} & \nf{NF} & \nd{ND} & \nd{ND} & \hr{27.3} & 12.2 & \nf{NF} & \nf{NF} & \nd{ND} & \nd{ND} \\
Solar-Decide & \hr{41.0} & \textbf{100.0} & \nf{NF} & \nf{NF} & \nd{ND} & \nd{ND} & \hr{28.6} & 61.1 & \nf{NF} & \nf{NF} & \nd{ND} & \nd{ND} \\
Kev-base & \ud{UD} & \ud{UD} & 14.6 & 88.2 & \nd{ND} & \nd{ND} & \ud{UD} & \ud{UD} & \nf{NF} & \nf{NF} & \nd{ND} & \nd{ND} \\
Kev-ft$^{*}$ & 3.6 & 97.1 & 12.5 & 99.6 & 11.1 & 99.6 & 5.1 & 71.0 & \hr{15.7} & \textbf{90.1} & 14.6 & 89.3 \\
\midrule
 & \multicolumn{6}{c}{\textbf{Reranking: SciDocs}} & \multicolumn{6}{c}{\textbf{Reranking: StackOverflow}} \\
\cmidrule(lr){2-7}\cmidrule(lr){8-13}
 & \multicolumn{2}{c}{Naive} & \multicolumn{2}{c}{Empirical} & \multicolumn{2}{c}{LTT} & \multicolumn{2}{c}{Naive} & \multicolumn{2}{c}{Empirical} & \multicolumn{2}{c}{LTT} \\
Model & Risk & Util & Risk & Util & Risk & Util & Risk & Util & Risk & Util & Risk & Util \\
\midrule
Reranker-Lite & \ud{UD} & \ud{UD} & 11.3 & 56.1 & \nd{ND} & \nd{ND} & \ud{UD} & \ud{UD} & \hr{18.5} & 57.3 & 14.0 & 51.7 \\
PCA-Ridge$^{*}$ & 7.6 & 24.6 & \hr{17.2} & 78.0 & 13.9 & 72.8 & 8.8 & 6.2 & \hr{16.5} & 23.2 & \nd{ND} & \nd{ND} \\
Jev & \hr{22.8} & 90.9 & \hr{18.8} & 85.8 & 13.7 & 76.8 & 8.1 & 25.1 & \hr{20.5} & 54.3 & \hr{16.6} & 44.3 \\
GPT-6-Luna & \hr{31.5} & \textbf{99.0} & \nf{NF} & \nf{NF} & \nd{ND} & \nd{ND} & 10.1 & 40.9 & \hr{15.9} & 55.3 & 11.3 & 42.5 \\
Span-01 & \hr{16.4} & 65.2 & \hr{15.9} & 63.4 & \nd{ND} & \nd{ND} & \hr{41.2} & 47.5 & \nf{NF} & \nf{NF} & \nd{ND} & \nd{ND} \\
Solar-Decide & \hr{40.1} & 98.0 & \nf{NF} & \nf{NF} & \nd{ND} & \nd{ND} & \hr{34.6} & \textbf{73.3} & \nf{NF} & \nf{NF} & \nd{ND} & \nd{ND} \\
Kev-base & \ud{UD} & \ud{UD} & \hr{19.4} & 25.4 & \nd{ND} & \nd{ND} & \ud{UD} & \ud{UD} & \nf{NF} & \nf{NF} & \nd{ND} & \nd{ND} \\
Kev-ft$^{*}$ & \hr{15.2} & 79.5 & \hr{15.9} & 83.9 & 14.8 & 78.5 & \hr{16.7} & 54.9 & \hr{16.2} & 53.9 & 13.4 & 50.3 \\
\bottomrule
\end{tabular*}
\caption{Test risk and utility (\%) at $\alpha=0.15$. Risk is FOR; utility is TNR (benign acceptance for safety, irrelevant removal for filtering). \hr{Red risk values} exceed the budget. \nd{ND}: no LTT certificate; \nf{NF}: no feasible empirical policy; \ud{UD}: undefined test FOR. --: not evaluated. *: supervised adaptation. Bold utility is the largest value per dataset.}\label{tab:policies}
\end{table*}

\begin{table*}[!t]
\centering\footnotesize
\setlength{\tabcolsep}{2pt}
\begin{tabular*}{\textwidth}{@{\extracolsep{\fill}}lcccccccccccc@{}}
\toprule
 & \multicolumn{2}{c}{\textbf{MD: WildGuardMix}} & \multicolumn{2}{c}{\textbf{MD: Aegis}} & \multicolumn{2}{c}{\textbf{SF: S-Labs}} & \multicolumn{2}{c}{\textbf{SF: ATBench}} & \multicolumn{2}{c}{\textbf{RK: SciDocs}} & \multicolumn{2}{c}{\textbf{RK: StackOverflow}} \\
\cmidrule(lr){2-3}\cmidrule(lr){4-5}\cmidrule(lr){6-7}\cmidrule(lr){8-9}\cmidrule(lr){10-11}\cmidrule(lr){12-13}
Model & Cost & Time & Cost & Time & Cost & Time & Cost & Time & Cost & Time & Cost & Time \\
\midrule
Qwen3Guard & 0.0071 & 231 & 0.0077 & 252 & 0.0065 & 213 & -- & -- & -- & -- & -- & -- \\
Reranker-Lite & -- & -- & -- & -- & -- & -- & -- & -- & 0.0013 & 224 & 0.0011 & 183 \\
PCA-Ridge$^{*}$ & 0.0026 & 5247 & 0.0039 & 1880 & 0.0013 & 5422 & 0.0215 & 14541 & 0.0014 & 5707 & 0.0012 & 3765 \\
Jev & 0.0216 & 383 & 0.0241 & 248 & 0.0204 & 379 & 0.0757 & 261 & 0.0175 & 364 & 0.0172 & 347 \\
GPT-6-Luna & 0.0375 & 320 & 0.0435 & 225 & 0.0347 & 186 & 0.1507 & 294 & 0.0283 & 207 & 0.0278 & 167 \\
Span-01 & 0.0084 & 672 & 0.0095 & 286 & 0.0081 & 648 & 0.0320 & 1059 & 0.0053 & 638 & 0.0050 & 710 \\
Solar-Decide & 0.0298 & 1361 & 0.0329 & 2076 & 0.0281 & 1038 & 0.0977 & 6303 & 0.0252 & 1536 & 0.0251 & 2351 \\
Kev-base & 0.0029 & 95 & 0.0030 & 99 & 0.0030 & 97 & 0.0076 & 249 & 0.0027 & 87 & 0.0026 & 85 \\
Kev-ft$^{*}$ & 0.0030 & 98 & 0.0032 & 104 & 0.0030 & 97 & 0.0074 & 241 & 0.0028 & 91 & 0.0027 & 88 \\
\bottomrule
\end{tabular*}
\caption{Cost (USD/1k) and latency (ms). MD: Moderation. SF: Safety. RK: Reranking. --: not evaluated.}\label{tab:efficiency-main}
\end{table*}
\section{Discussion and Conclusion}
\label{sec:conclusion}
Deploying System One models for asymmetric decisions, where one error type carries severe consequences, demands more than aggregate calibration. Results show that favorable global diagnostics do not guarantee utility, and naive or empirical thresholding often lead to severe, silent risk violations. 

LTT offers a rigorous statistical alternative, controlling the useful action region without trusting raw scores as true probabilities. The operational price of this safety is reduced policy availability. LTT minimizes budget exceedances by returning non-deployment when a certificate cannot be established, correctly defaulting to a conservative fallback (e.g., blocking a prompt or retaining a document) rather than operating unsafely. However, our findings show that this availability bottleneck can be mitigated: when general interfaces fail to yield a certificate, adapting lightweight, local open-weights models provides a highly reliable and cost-efficient pathway to recover certified policies, though this comes with the cost of training data.

Future work should extend LTT certification to multi-model cascades to optimize the balance between safety, latency, and inference cost, and explore how larger calibration samples tighten risk bounds to increase policy availability.
\section*{Limitations}
First, the finite-sample guarantee in Theorem~\ref{thm:ltt} relies on calibration and deployment examples being drawn i.i.d.\ from the same distribution $P$. Our six English-language benchmarks do not establish performance under deployment or adversarial shift, where maintaining risk control would require periodic recalibration or distribution-shift adaptations. Moreover, ATBench evaluates retrospective classification of truncated trajectories, and our reranking tasks use a constructed document mixture rather than an end-to-end retrieval pipeline; both choices affect the meaning and transfer of the measured FOR.

Second, exact binomial certification imposes a fundamental sample-size requirement on the calibration set. Even with zero observed errors ($K_j=0$), certifying a threshold at budget $\alpha$ and error level $\delta$ requires at least $\lceil\log\delta/\log(1-\alpha)\rceil$ selected calibration examples (Appendix~\ref{app:ltt_impl}). Consequently, when calibration sets are small or risk budgets are strict (e.g., $\alpha=0.05$), LTT frequently defaults to non-deployment ($\bot$), trading policy availability for safety.

Finally, while adapting a lightweight open-weights model (Kev-ft) reliably recovers certified utility when zero-shot System One interfaces fail to certify, this pathway requires assembling labeled training data ($D_{\mathrm{fit}}$) and local training compute, which may not be feasible in cold-start deployments.

\section*{Ethical Considerations}
This study evaluates datasets containing harmful, offensive, and adversarial text. We analyze dataset labels and model scores rather than distributing new attack payloads, and this paper contains no raw harmful examples. 

A low measured FOR is not a blanket safety endorsement: false negatives expose users to harm, while false positives suppress benign expression. Aggregate metrics do not establish equitable error rates across demographic groups or topics. In filtering, document removal can skew downstream generation even at a low discarded-set contamination rate. Safe deployment requires an application-specific risk definition, continuous monitoring, and an explicit fallback action when LTT returns no valid policy ($\bot$).

Benchmark annotations embody specific normative safety and relevance judgments that should not be mistaken for universal definitions. Furthermore, this study cannot guarantee the absence of training-data contamination for externally supplied APIs. Model and data identifiers are provided to support independent scrutiny; any redistribution must comply with the original licenses.

\bibliography{references}
\clearpage
\appendix
\onecolumn
\section{Calibration--Policy Identity Supporting Proof}\label{app:proofs}

\begin{proof}[Proof of Proposition~\ref{prop:bridge}]
For $\mu(\tau)>0$, the event $\{S<\tau\}$ is measurable with respect to $S$. Iterated expectation gives
\begin{equation*}
\begin{split}
R(\tau) &= \E[Y\mid S<\tau] \\
        &= \E[m(S)\mid S<\tau] \\
        &= \E[S\mid S<\tau]+\E[m(S)-S\mid S<\tau].
\end{split}
\end{equation*}

The first term is at most $\tau$, and the second is at most
\begin{align*}
&\E[|m(S)-S|\mid S<\tau] \\
&\quad =\frac{\E[|m(S)-S|\ind\{S<\tau\}]}{\mu(\tau)}.
\end{align*}
This proves the bound. Under exact calibration $m(S)=S$ almost surely, the discrepancy vanishes and $R(\tau)\le\tau$.
\end{proof}

\section{Detailed Experimental Protocol}
\label{app:protocol}

\subsection{Model Configurations and Supervision}\label{app:models}
\paragraph{System One APIs and Specialists.} The evaluated proprietary interfaces are \texttt{typesafe/jev-1.13}, \texttt{respan/span-01}, \texttt{upstage/solar-decide}, and \texttt{openai/gpt-6-luna-decisions}. The filtering specialist relies on \texttt{voyageai/rerank-2.5-lite}, whose relevance score is clipped to $[0,1]$. 

\paragraph{Local Open-Weights Models.} 
Kev-base initializes from \texttt{jaredpalmer/kev-0.8b}. Kev-ft configuration warm-starts its rank-16 LoRA adapters and 256-dimensional pointer head on a frozen \path{Qwen/Qwen3.5-0.8B-Base} backbone. Training uses two epochs, seed 0, AdamW with weight decay 0.01, and a OneCycle schedule with peak learning rate $2\times10^{-5}$ for both adapters and head and a 10\% increasing phase. The microbatch is 2 with four accumulation steps on one device (effective batch 8); gradient norms are clipped at 1.0. Computation uses bfloat16 with float32 trainable weights, gradient checkpointing, and a maximum state length of 2,048 tokens. The wrapper truncates inputs to fit the training context. The final checkpoint after two epochs is used.

The pooled training data contains 9,988 records: 1,600 each from WildGuardMix, S-Labs, SciDocs, StackOverflow, and BeaverTails; 1,588 from Aegis; and 400 from ATBench. Thus adaptation also uses BeaverTails, which is not an evaluation benchmark in this paper. The generator draws from fit pools, deduplicates exact rendered states across sources, and reserves 20\% per source for a separate held-out diagnostic. 

\paragraph{PCA-Ridge Baseline.} 
This baseline extracts instruction-templated embeddings using \path{qwen/qwen3-embedding-8b}, applies L2 normalization and a 100-dimensional PCA projection, and fits a logistic regression classifier ($C=1.0$, solver \texttt{lbfgs}, and 2,000 maximum iterations). Both PCA and the classifier are fit independently on each dataset's specific fit split.

\paragraph{Qwen3Guard Score Extraction.}\label{app:qwen}
Qwen3Guard (\texttt{Qwen/Qwen3Guard-Gen-0.6B}) emits specific categorical tokens rather than direct probabilities. For moderation tasks, we compute the score as the full-vocabulary softmax mass on Unsafe token variants following the \texttt{Safety:} marker.

\subsection{Data Splitting and Evaluation Roles}\label{app:data}
We partition datasets into disjoint Fit, Calibration, and Test pools to strictly separate model training, threshold selection, and final evaluation. Fit data is exclusively used for model fine-tuning (Kev-ft), baseline fitting (PCA-Ridge), and constructing the candidate threshold grid ($\mathcal{T}$). Calibration data ($D_{\mathrm{cal}}$) is used strictly for threshold selection (LTT and Empirical), and Test data evaluates the resulting fixed policies.

Calibration and Test splits use label-stratified sampling. Safety/moderation tasks also group repeated prompts or texts during overlap screening and partitioning. For reranking tasks (SciDocs and StackOverflow), we construct balanced query--candidate pairs, fit and evaluation pools are separated by query ID, and select one candidate per query using a deterministic query hash to choose a relevant or irrelevant candidate. This avoids counting multiple candidates for one query as independent rows.

\subsection{Input Formatting Specifications}\label{app:data2}
Input rendering depends on the task domain. Moderation queries evaluate solely the user's prompt. For safety evaluation (ATBench), the model evaluates serialized agent trajectories encompassing actions, disclosures, and environmental thoughts. To maintain context limits, these trajectories are bounded to 600 characters for tool names, 1,200 characters per turn, and 12,000 characters overall (with truncation markers appended where applicable). 

\subsection{LTT Implementation Details}\label{app:lttimp}
\label{app:ltt_impl}
The candidate threshold family $\mathcal{T}$ is constructed from up to 200 sorted fit-score midpoints on $D_{\mathrm{fit}}$, plus the endpoints $\{0, 1\}$. Because $\mathcal{T}$ is defined strictly on $D_{\mathrm{fit}}$, the candidate testing sequence is fixed independently of $D_{\mathrm{cal}}$ \citep[Section~2.3.1]{angelopoulos2025learn}. Under the rule $s(X) < \tau$, a threshold of $0$ selects an empty set, while $1$ selects all scores strictly below one.

Following the selective classification protocol of \citet[Section~3.2]{angelopoulos2025learn}, we compute $p$-values on the selected calibration subset using the exact binomial tail bound. Even when zero errors are observed ($K_j=0$), the binomial $p$-value reduces to $p_j = (1-\alpha)^{N_j}$. Satisfying the certification condition $p_j \leq \delta$ therefore requires $(1-\alpha)^{N_j} \leq \delta$; taking logarithms and dividing by $\log(1-\alpha) < 0$ reverses the inequality and yields the minimum certifiable support size:
\begin{equation}
    N_j \geq \left\lceil\frac{\log\delta}{\log(1-\alpha)}\right\rceil.
    \label{eq:sample}
\end{equation}
At $\delta=0.10$, Equation~\ref{eq:sample} requires at least 45 selected calibration examples for $\alpha=0.05$, and 22 for $\alpha=0.10$. Following the low-support clipping convention of \citet[Section~3.2]{angelopoulos2025learn} ($n(\mathcal{T}_\lambda) \geq 25$), we screen out candidates with calibration support below $\max(25, \lceil\log\delta/\log(1-\alpha)\rceil)$ so that uncertifiable low-support thresholds near $\tau=0$ do not prematurely halt fixed-sequence testing. Our Empirical baseline ($\widehat{R}_{\mathrm{cal},j} \leq \alpha$) follows \citet[Supplement Section~1.1]{angelopoulos2025learn} on the same grid with a support floor of 25.

\subsection{Metrics and Uncertainty Quantification}\label{app:metric-details}
Performance metrics are defined standardly based on the confusion matrix elements (TN, FN, FP, TP). The target risk is the False Omission Rate: $\mathrm{FOR} = \mathrm{FN}/(\mathrm{TN}+\mathrm{FN})$. The primary utility metric is Benign Acceptance (TNR) for safety tasks and Irrelevant Removal for filtering tasks, computed as $\mathrm{TN}/(\mathrm{TN}+\mathrm{FP})$. Positive excess risk is computed in percentage points as $e = \max(\widehat R-\alpha, 0)$. 

Calibration is measured via the Brier Score ($\widehat{\mathrm{BS}}$) and Expected Calibration Error ($\widehat{\mathrm{ECE}}$). For a test sample of size $n$, partitioned into 10 equal-width bins $b$ with respective sizes $|I_b|$, mean scores $\bar s_b$, and positive label frequencies $\bar y_b$:

\[
\widehat{\mathrm{ECE}}=\sum_b\frac{|I_b|}{n}|\bar y_b-\bar s_b|,
\]
\[
\widehat{\mathrm{BS}}=\frac{1}{n}\sum_i(s_i-Y_i)^2.
\]

Uncertainty is quantified using 1,000 unstratified bootstrap resamples (seed 42) applied to the test rows, evaluating fixed thresholds and models. Confidence intervals are reported using the 2.5th and 97.5th percentiles via linear interpolation. 

\subsection{Resource Accounting}
\label{app:efficiency}
Cost and latency metrics trace the fit, calibration, and test inference workloads. Local inference is estimated at a standard rate of \$0.11 per GPU-hour, while proprietary API costs use respective provider accounting. Because API pricing structures can change, we note that all experiments and cost calculations were conducted between late September and early October 2026. Costs are normalized per 1,000 workload examples; latency is reported in milliseconds per example. These figures represent recorded historical overheads rather than strictly controlled, hardware-synchronized microbenchmarks, and are intended to provide descriptive operational context.

\subsection{Artifact Access and Terms}
\label{app:artifacts}
The public dataset cards list \href{https://huggingface.co/datasets/allenai/wildguardmix}{WildGuardMix} under ODC-BY (with AI2 Responsible Use Guidelines), \href{https://huggingface.co/datasets/nvidia/Aegis-AI-Content-Safety-Dataset-1.0}{Aegis 1.0} under CC-BY-4.0, \href{https://huggingface.co/datasets/S-Labs/prompt-injection-dataset}{S-Labs} under MIT, and \href{https://huggingface.co/datasets/AI45Research/ATBench}{ATBench} under Apache-2.0. The additional training source \href{https://huggingface.co/datasets/PKU-Alignment/BeaverTails}{BeaverTails} lists CC-BY-NC-4.0. \href{https://huggingface.co/datasets/mteb/scidocs-reranking}{SciDocs} and \href{https://huggingface.co/datasets/mteb/stackoverflowdupquestions-reranking}{StackOverflow} derivative cards did not expose an explicit license declaration; this is a documentation gap, not permission for unrestricted redistribution. Original-source and derivative terms must be considered.

The open-weights models (\href{https://huggingface.co/Qwen/Qwen3Guard-Gen-0.6B}{Qwen3Guard-Gen-0.6B}, \href{https://huggingface.co/Qwen/Qwen3-Embedding-8B}{Qwen3-Embedding-8B}, and \href{https://huggingface.co/jaredpalmer/kev-0.8b}{Kev-0.8B}) are released under the Apache-2.0 license, while the proprietary endpoints (Jev, GPT-6-Luna, Span-01, Solar-Decide, and Voyage Rerank 2.5 Lite) were accessed under their respective providers' standard commercial terms of service. All experiments were conducted strictly as offline benchmark evaluations, which is consistent with the intended research use of these artifacts and respects model-card restrictions against unattended high-stakes deployment.

\section{Additional Score and Policy Results}\label{app:full-results}
\subsection{Score Diagnostics with 95\% Intervals}

\begin{table*}[b!]
\centering\footnotesize
\setlength{\tabcolsep}{3pt}
\renewcommand{\arraystretch}{1.04}
\begin{tabular*}{\textwidth}{@{\extracolsep{\fill}}llcccc@{}}
\toprule
\textbf{Task / Dataset} & \textbf{Model} & ECE $\downarrow$ & Brier $\downarrow$ & AUROC $\uparrow$ & AP$_0$ $\uparrow$ \\
\midrule
\textbf{Moderation} 
& Qwen3Guard & 0.085 [0.068, 0.106] & 0.090 [0.073, 0.108] & 0.950 [0.935, 0.965] & 0.947 [0.924, 0.967] \\
WildGuardMix 
& PCA-Ridge$^{*}$ & 0.042 [0.034, 0.067] & 0.091 [0.078, 0.103] & 0.946 [0.931, 0.959] & 0.956 [0.940, 0.968] \\
& Jev & 0.034 [0.028, 0.056] & 0.073 [0.063, 0.084] & 0.963 [0.951, 0.973] & 0.969 [0.957, 0.979] \\
& GPT-6-Luna & 0.103 [0.084, 0.125] & 0.108 [0.089, 0.126] & 0.925 [0.906, 0.942] & 0.907 [0.880, 0.931] \\
& Span-01 & 0.189 [0.166, 0.224] & 0.238 [0.216, 0.261] & 0.765 [0.732, 0.799] & 0.746 [0.701, 0.795] \\
& Solar-Decide & 0.104 [0.081, 0.128] & 0.130 [0.108, 0.150] & 0.893 [0.867, 0.917] & 0.855 [0.814, 0.900] \\
& Kev-base & 0.121 [0.097, 0.149] & 0.186 [0.176, 0.196] & 0.835 [0.809, 0.862] & 0.863 [0.835, 0.891] \\
\bottomrule
\end{tabular*}
\caption{Test score diagnostics with 95\% intervals [lower, upper]. *: supervised adaptation.}\label{tab:quality-ci-all}
\end{table*}
\clearpage
\subsection{Risk and Utility at Every Budget, with 95\% Intervals}
These tables extend the main Table~\ref{tab:policies} layout to all $\alpha$ with its intervals. A zero-width interval after no errors does not establish zero population risk.
\begin{table*}[b!]
\centering\footnotesize
\setlength{\tabcolsep}{2pt}
\renewcommand{\arraystretch}{0.94}
\begin{tabularx}{\textwidth}{@{}lc*{6}{>{\centering\arraybackslash}X}@{}}
\toprule
 &  & \multicolumn{3}{c}{\textbf{Moderation: WildGuardMix}} & \multicolumn{3}{c}{\textbf{Moderation: Aegis}} \\
\cmidrule(lr){3-5}\cmidrule(lr){6-8}
Model & Metric & Naive & Empirical & LTT & Naive & Empirical & LTT \\
\midrule
\rowcolor{gray!12}Qwen3Guard & Risk & \hr{9.2}\hspace{1pt}[6,13] & \hr{6.2}\hspace{1pt}[3,9] & 3.3\hspace{1pt}[1,7] & \hr{32.4}\hspace{1pt}[26,39] & \hr{5.9}\hspace{1pt}[1,11] & \nd{ND} \\
 & Utility & 91.3\hspace{1pt}[88,94] & 72.9\hspace{1pt}[69,78] & 42.7\hspace{1pt}[38,48] & 89.5\hspace{1pt}[85,94] & 50.3\hspace{1pt}[43,58] & \nd{ND} \\
\rowcolor{gray!12}PCA-Ridge$^{*}$ & Risk & 0.0\hspace{1pt}[0,0] & \hr{5.1}\hspace{1pt}[3,8] & 3.2\hspace{1pt}[1,6] & \ud{UD} & \hr{8.7}\hspace{1pt}[3,14] & \nd{ND} \\
 & Utility & 18.8\hspace{1pt}[15,23] & 70.6\hspace{1pt}[66,75] & 63.0\hspace{1pt}[58,68] & \ud{UD} & 60.7\hspace{1pt}[53,68] & \nd{ND} \\
\rowcolor{gray!12}Jev & Risk & 0.0\hspace{1pt}[0,0] & \hr{6.6}\hspace{1pt}[4,10] & 4.4\hspace{1pt}[2,7] & \hr{7.7}\hspace{1pt}[1,16] & \hr{7.7}\hspace{1pt}[1,16] & \nd{ND} \\
 & Utility & 37.6\hspace{1pt}[33,42] & 83.3\hspace{1pt}[80,87] & 77.8\hspace{1pt}[74,82] & 25.1\hspace{1pt}[18,32] & 25.1\hspace{1pt}[18,32] & \nd{ND} \\
\bottomrule
\end{tabularx}
\caption{Test risk and utility at $\alpha=0.05$ (\%). Shaded Risk rows report FOR; white Utility rows report TNR. Entries are estimates [95\% CI], with bounds rounded outward to whole percentage points. Red risk exceeds the budget. \nd{ND}/\nf{NF}: no returned policy; \ud{UD}: undefined FOR (zero TNR). --: not evaluated. *: supervised adaptation.}\label{tab:paired-ci-0.05}
\end{table*}

\begin{table*}[h]
\centering\footnotesize
\setlength{\tabcolsep}{2pt}
\renewcommand{\arraystretch}{0.94}
\begin{tabularx}{\textwidth}{@{}lc*{6}{>{\centering\arraybackslash}X}@{}}
\toprule
 &  & \multicolumn{3}{c}{\textbf{Moderation: WildGuardMix}} & \multicolumn{3}{c}{\textbf{Moderation: Aegis}} \\
\cmidrule(lr){3-5}\cmidrule(lr){6-8}
Model & Metric & Naive & Empirical & LTT & Naive & Empirical & LTT \\
\midrule
\rowcolor{gray!12}Qwen3Guard & Risk & 9.7\hspace{1pt}[7,13] & 9.4\hspace{1pt}[7,13] & 7.8\hspace{1pt}[5,11] & \hr{36.4}\hspace{1pt}[30,42] & \hr{11.5}\hspace{1pt}[6,17] & 5.9\hspace{1pt}[1,11] \\
 & Utility & 94.1\hspace{1pt}[91,97] & 93.2\hspace{1pt}[90,96] & 87.9\hspace{1pt}[85,91] & 91.6\hspace{1pt}[87,96] & 60.7\hspace{1pt}[53,68] & 50.3\hspace{1pt}[43,58] \\
\rowcolor{gray!12}PCA-Ridge$^{*}$ & Risk & 2.0\hspace{1pt}[0,4] & 9.8\hspace{1pt}[7,13] & 8.3\hspace{1pt}[5,12] & 0.0\hspace{1pt}[0,0] & 9.5\hspace{1pt}[4,15] & \nd{ND} \\
 & Utility & 51.6\hspace{1pt}[46,57] & 83.7\hspace{1pt}[80,88] & 79.7\hspace{1pt}[76,84] & 14.7\hspace{1pt}[9,20] & 64.9\hspace{1pt}[58,72] & \nd{ND} \\
\rowcolor{gray!12}Jev & Risk & 3.3\hspace{1pt}[1,6] & \hr{11.9}\hspace{1pt}[9,15] & \hr{10.5}\hspace{1pt}[7,14] & \hr{11.7}\hspace{1pt}[6,18] & \hr{11.5}\hspace{1pt}[6,18] & \nd{ND} \\
 & Utility & 56.4\hspace{1pt}[51,61] & 95.3\hspace{1pt}[93,98] & 93.4\hspace{1pt}[90,96] & 67.0\hspace{1pt}[60,75] & 60.7\hspace{1pt}[53,68] & \nd{ND} \\
\bottomrule
\end{tabularx}
\caption{Test risk and utility at $\alpha=0.10$ (\%). Shaded Risk rows report FOR; white Utility rows report TNR. Entries are estimates [95\% CI], with bounds rounded outward to whole percentage points. Red risk exceeds the budget. \nd{ND}/\nf{NF}: no returned policy; \ud{UD}: undefined FOR (zero TNR). --: not evaluated. *: supervised adaptation.}\label{tab:paired-ci-0.10}
\end{table*}

\begin{table*}[h]
\centering\footnotesize
\setlength{\tabcolsep}{2pt}
\renewcommand{\arraystretch}{0.94}
\begin{tabularx}{\textwidth}{@{}lc*{6}{>{\centering\arraybackslash}X}@{}}
\toprule
 &  & \multicolumn{3}{c}{\textbf{Moderation: WildGuardMix}} & \multicolumn{3}{c}{\textbf{Moderation: Aegis}} \\
\cmidrule(lr){3-5}\cmidrule(lr){6-8}
Model & Metric & Naive & Empirical & LTT & Naive & Empirical & LTT \\
\midrule
\rowcolor{gray!12}Qwen3Guard & Risk & 10.8\hspace{1pt}[8,14] & 12.2\hspace{1pt}[9,16] & 11.4\hspace{1pt}[8,15] & \hr{38.4}\hspace{1pt}[32,44] & \hr{18.2}\hspace{1pt}[12,24] & 11.5\hspace{1pt}[6,17] \\
 & Utility & 94.3\hspace{1pt}[92,97] & 95.8\hspace{1pt}[93,98] & 95.1\hspace{1pt}[93,98] & 93.2\hspace{1pt}[89,97] & 72.8\hspace{1pt}[66,79] & 60.7\hspace{1pt}[53,68] \\
\rowcolor{gray!12}PCA-Ridge$^{*}$ & Risk & 4.2\hspace{1pt}[2,7] & 13.1\hspace{1pt}[9,16] & 11.7\hspace{1pt}[8,15] & 1.4\hspace{1pt}[0,5] & \hr{17.3}\hspace{1pt}[11,24] & 9.5\hspace{1pt}[4,15] \\
 & Utility & 67.4\hspace{1pt}[63,72] & 91.3\hspace{1pt}[88,94] & 87.9\hspace{1pt}[84,91] & 36.1\hspace{1pt}[29,43] & 72.8\hspace{1pt}[66,80] & 64.9\hspace{1pt}[58,72] \\
\rowcolor{gray!12}Jev & Risk & 3.3\hspace{1pt}[1,6] & \hr{16.2}\hspace{1pt}[13,20] & 12.8\hspace{1pt}[9,16] & \hr{20.8}\hspace{1pt}[15,27] & \hr{23.1}\hspace{1pt}[17,30] & 11.7\hspace{1pt}[6,18] \\
 & Utility & 68.7\hspace{1pt}[64,74] & 98.1\hspace{1pt}[96,100] & 96.8\hspace{1pt}[95,99] & 83.8\hspace{1pt}[78,89] & 85.3\hspace{1pt}[80,91] & 67.0\hspace{1pt}[60,75] \\
\bottomrule
\end{tabularx}
\caption{Test risk and utility at $\alpha=0.15$ (\%). Shaded Risk rows report FOR; white Utility rows report TNR. Entries are estimates [95\% CI], with bounds rounded outward to whole percentage points. Red risk exceeds the budget. \nd{ND}/\nf{NF}: no returned policy; \ud{UD}: undefined FOR (zero TNR). --: not evaluated. *: supervised adaptation.}\label{tab:paired-ci-0.15}
\end{table*}

\begin{table*}[h]
\centering\footnotesize
\setlength{\tabcolsep}{2pt}
\renewcommand{\arraystretch}{0.94}
\begin{tabularx}{\textwidth}{@{}lc*{6}{>{\centering\arraybackslash}X}@{}}
\toprule
 &  & \multicolumn{3}{c}{\textbf{Moderation: WildGuardMix}} & \multicolumn{3}{c}{\textbf{Moderation: Aegis}} \\
\cmidrule(lr){3-5}\cmidrule(lr){6-8}
Model & Metric & Naive & Empirical & LTT & Naive & Empirical & LTT \\
\midrule
\rowcolor{gray!12}Qwen3Guard & Risk & 11.3\hspace{1pt}[8,15] & 16.7\hspace{1pt}[13,20] & 14.0\hspace{1pt}[11,17] & \hr{40.5}\hspace{1pt}[34,47] & \hr{28.6}\hspace{1pt}[22,35] & 19.4\hspace{1pt}[13,26] \\
 & Utility & 94.9\hspace{1pt}[92,97] & 97.0\hspace{1pt}[95,99] & 96.4\hspace{1pt}[94,99] & 93.7\hspace{1pt}[90,98] & 86.4\hspace{1pt}[81,91] & 73.8\hspace{1pt}[67,80] \\
\rowcolor{gray!12}PCA-Ridge$^{*}$ & Risk & 5.6\hspace{1pt}[3,9] & 17.9\hspace{1pt}[14,21] & 15.7\hspace{1pt}[12,19] & 5.6\hspace{1pt}[1,11] & \hr{23.8}\hspace{1pt}[17,30] & 17.4\hspace{1pt}[11,24] \\
 & Utility & 74.6\hspace{1pt}[70,79] & 96.2\hspace{1pt}[94,98] & 93.9\hspace{1pt}[91,97] & 52.9\hspace{1pt}[45,61] & 80.6\hspace{1pt}[74,87] & 74.3\hspace{1pt}[68,81] \\
\rowcolor{gray!12}Jev & Risk & 3.8\hspace{1pt}[2,6] & 18.4\hspace{1pt}[15,22] & 17.8\hspace{1pt}[14,22] & \hr{25.6}\hspace{1pt}[19,32] & \hr{25.8}\hspace{1pt}[20,32] & \hr{23.5}\hspace{1pt}[18,30] \\
 & Utility & 74.2\hspace{1pt}[70,79] & 98.3\hspace{1pt}[97,100] & 98.3\hspace{1pt}[97,100] & 86.9\hspace{1pt}[82,92] & 87.4\hspace{1pt}[82,93] & 85.3\hspace{1pt}[80,91] \\
\bottomrule
\end{tabularx}
\caption{Test risk and utility at $\alpha=0.20$ (\%). Shaded Risk rows report FOR; white Utility rows report TNR. Entries are estimates [95\% CI], with bounds rounded outward to whole percentage points. Red risk exceeds the budget. \nd{ND}/\nf{NF}: no returned policy; \ud{UD}: undefined FOR (zero TNR). --: not evaluated. *: supervised adaptation.}\label{tab:paired-ci-0.20}
\end{table*}

\begin{table*}[h]
\centering\footnotesize
\setlength{\tabcolsep}{2pt}
\renewcommand{\arraystretch}{0.94}
\begin{tabularx}{\textwidth}{@{}lc*{6}{>{\centering\arraybackslash}X}@{}}
\toprule
 &  & \multicolumn{3}{c}{\textbf{Moderation: WildGuardMix}} & \multicolumn{3}{c}{\textbf{Moderation: Aegis}} \\
\cmidrule(lr){3-5}\cmidrule(lr){6-8}
Model & Metric & Naive & Empirical & LTT & Naive & Empirical & LTT \\
\midrule
\rowcolor{gray!12}Qwen3Guard & Risk & 12.4\hspace{1pt}[9,16] & \hr{40.9}\hspace{1pt}[37,45] & 39.0\hspace{1pt}[35,43] & \hr{44.7}\hspace{1pt}[38,51] & \hr{44.2}\hspace{1pt}[38,50] & \hr{40.3}\hspace{1pt}[34,46] \\
 & Utility & 96.0\hspace{1pt}[94,98] & 99.8\hspace{1pt}[99,100] & 99.8\hspace{1pt}[99,100] & 96.3\hspace{1pt}[93,99] & 95.8\hspace{1pt}[92,99] & 93.2\hspace{1pt}[89,97] \\
\rowcolor{gray!12}PCA-Ridge$^{*}$ & Risk & 10.8\hspace{1pt}[7,14] & \hr{40.1}\hspace{1pt}[36,44] & 37.7\hspace{1pt}[34,41] & 13.7\hspace{1pt}[8,20] & \hr{42.0}\hspace{1pt}[36,48] & 39.5\hspace{1pt}[33,45] \\
 & Utility & 87.3\hspace{1pt}[84,91] & 99.6\hspace{1pt}[98,100] & 99.4\hspace{1pt}[98,100] & 69.1\hspace{1pt}[62,76] & 94.8\hspace{1pt}[91,98] & 94.8\hspace{1pt}[91,98] \\
\rowcolor{gray!12}Jev & Risk & 7.1\hspace{1pt}[4,10] & \hr{40.4}\hspace{1pt}[36,45] & 36.4\hspace{1pt}[32,41] & 35.0\hspace{1pt}[29,41] & \hr{44.0}\hspace{1pt}[38,50] & \hr{40.4}\hspace{1pt}[34,46] \\
 & Utility & 87.9\hspace{1pt}[85,91] & 100.0\hspace{1pt}[100,100] & 100.0\hspace{1pt}[100,100] & 95.3\hspace{1pt}[92,99] & 98.4\hspace{1pt}[96,100] & 97.9\hspace{1pt}[95,100] \\
\bottomrule
\end{tabularx}
\caption{Test risk and utility at $\alpha=0.40$ (\%). Shaded Risk rows report FOR; white Utility rows report TNR. Entries are estimates [95\% CI], with bounds rounded outward to whole percentage points. Red risk exceeds the budget. \nd{ND}/\nf{NF}: no returned policy; \ud{UD}: undefined FOR (zero TNR). --: not evaluated. *: supervised adaptation.}\label{tab:paired-ci-0.40}
\end{table*}

\begin{table*}[h]
\centering\footnotesize
\setlength{\tabcolsep}{2pt}
\renewcommand{\arraystretch}{0.94}
\begin{tabularx}{\textwidth}{@{}lc*{6}{>{\centering\arraybackslash}X}@{}}
\toprule
 &  & \multicolumn{3}{c}{\textbf{Moderation: WildGuardMix}} & \multicolumn{3}{c}{\textbf{Moderation: Aegis}} \\
\cmidrule(lr){3-5}\cmidrule(lr){6-8}
Model & Metric & Naive & Empirical & LTT & Naive & Empirical & LTT \\
\midrule
\rowcolor{gray!12}Qwen3Guard & Risk & 13.3\hspace{1pt}[10,17] & 44.4\hspace{1pt}[41,48] & 44.4\hspace{1pt}[41,48] & 46.6\hspace{1pt}[41,52] & \hr{53.9}\hspace{1pt}[48,59] & 49.5\hspace{1pt}[44,55] \\
 & Utility & 96.4\hspace{1pt}[94,99] & 100.0\hspace{1pt}[100,100] & 100.0\hspace{1pt}[100,100] & 97.9\hspace{1pt}[95,100] & 99.5\hspace{1pt}[98,100] & 99.0\hspace{1pt}[97,100] \\
\rowcolor{gray!12}PCA-Ridge$^{*}$ & Risk & 12.1\hspace{1pt}[9,15] & 44.4\hspace{1pt}[41,48] & 44.4\hspace{1pt}[41,48] & 19.7\hspace{1pt}[13,26] & \hr{52.8}\hspace{1pt}[47,58] & 48.8\hspace{1pt}[43,54] \\
 & Utility & 90.5\hspace{1pt}[87,94] & 100.0\hspace{1pt}[100,100] & 100.0\hspace{1pt}[100,100] & 74.9\hspace{1pt}[68,81] & 99.0\hspace{1pt}[97,100] & 98.4\hspace{1pt}[96,100] \\
\rowcolor{gray!12}Jev & Risk & 8.8\hspace{1pt}[6,12] & 44.4\hspace{1pt}[41,48] & 44.4\hspace{1pt}[41,48] & 38.3\hspace{1pt}[32,44] & \hr{53.8}\hspace{1pt}[48,59] & \hr{51.8}\hspace{1pt}[46,57] \\
 & Utility & 91.8\hspace{1pt}[89,95] & 100.0\hspace{1pt}[100,100] & 100.0\hspace{1pt}[100,100] & 97.9\hspace{1pt}[95,100] & 99.0\hspace{1pt}[97,100] & 98.4\hspace{1pt}[96,100] \\
\bottomrule
\end{tabularx}
\caption{Test risk and utility at $\alpha=0.50$ (\%). Shaded Risk rows report FOR; white Utility rows report TNR. Entries are estimates [95\% CI], with bounds rounded outward to whole percentage points. Red risk exceeds the budget. \nd{ND}/\nf{NF}: no returned policy; \ud{UD}: undefined FOR (zero TNR). --: not evaluated. *: supervised adaptation.}\label{tab:paired-ci-0.50}
\end{table*}

\clearpage
\subsection{Risk Versus Budget: All Models and Datasets}\label{app:risk-budget-panels}
Each figure shows all evaluated models; status grids have $\alpha$ columns and Naive/Empirical/LTT rows.
\begin{figure*}[ht]
\centering
\includegraphics[width=\textwidth]{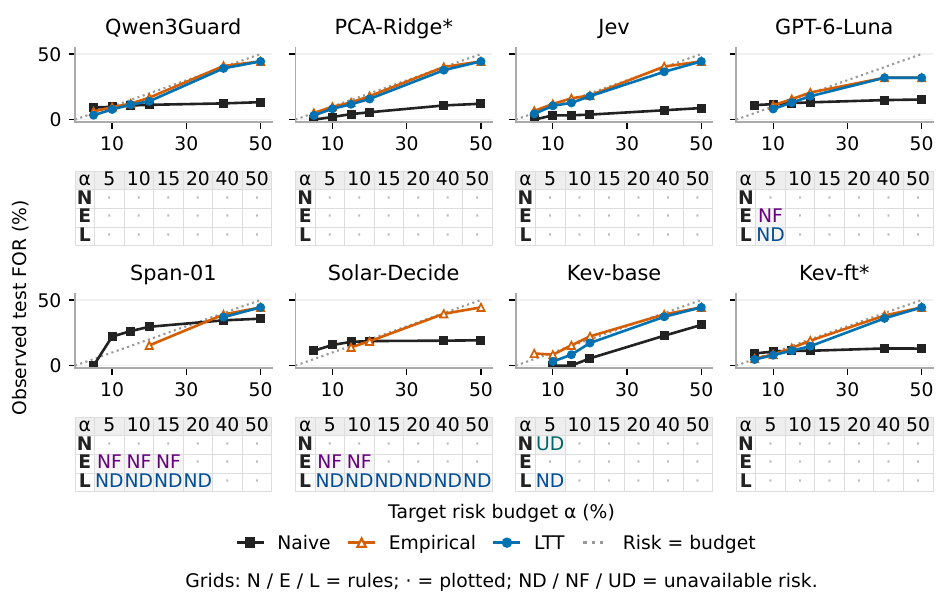}
\caption{WildGuardMix: test FOR versus budget risk.}
\label{fig:budget-risk-full}
\end{figure*}
\begin{figure*}[b!]
\centering
\includegraphics[width=\textwidth]{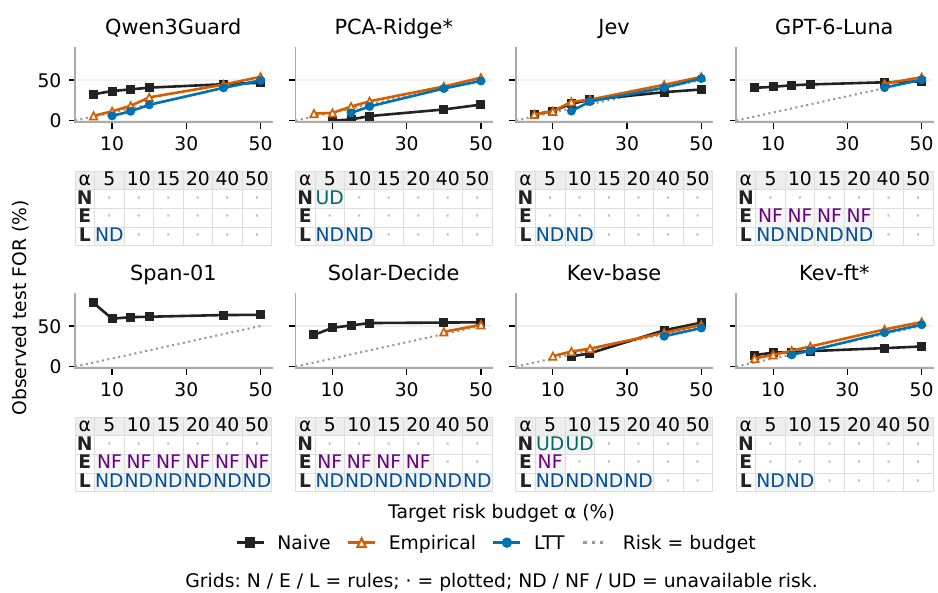}
\caption{Aegis: test FOR versus risk budget.}
\end{figure*}
\begin{figure*}[ht]
\centering
\includegraphics[width=\textwidth]{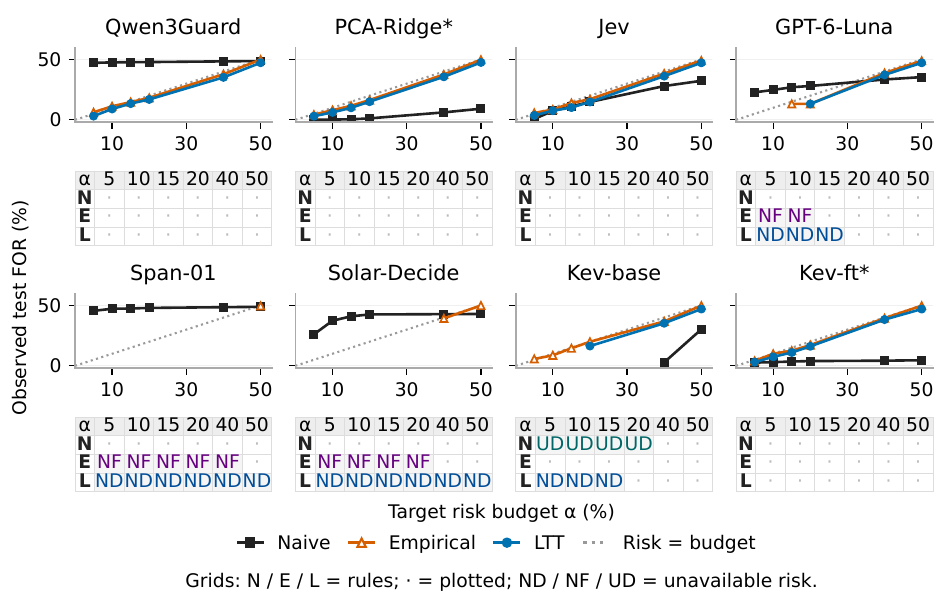}
\caption{S-Labs: test FOR versus risk budget.}
\end{figure*}
\begin{figure*}[ht]
\centering
\includegraphics[width=\textwidth]{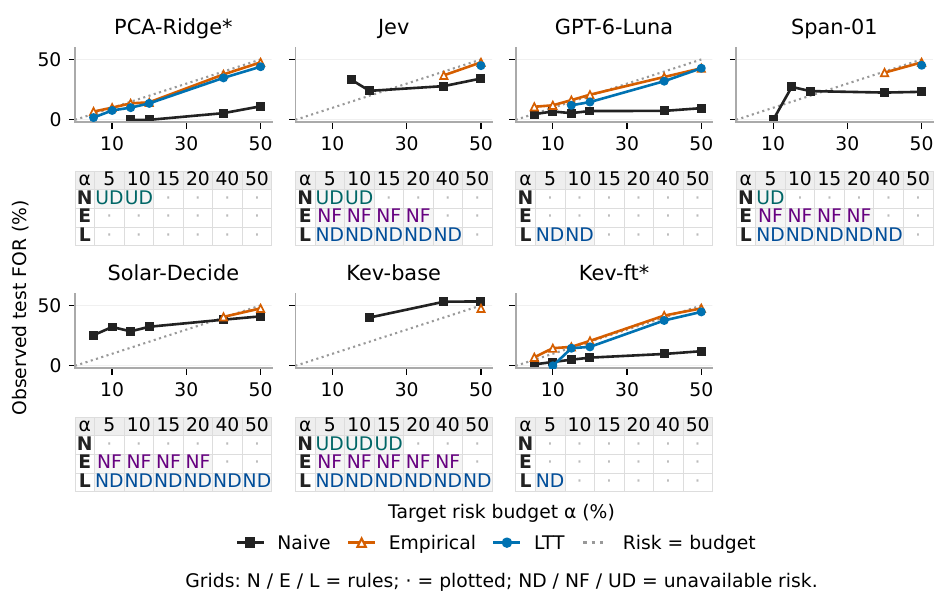}
\caption{ATBench: test FOR versus risk budget.}
\end{figure*}
\begin{figure*}[ht]
\centering
\includegraphics[width=\textwidth]{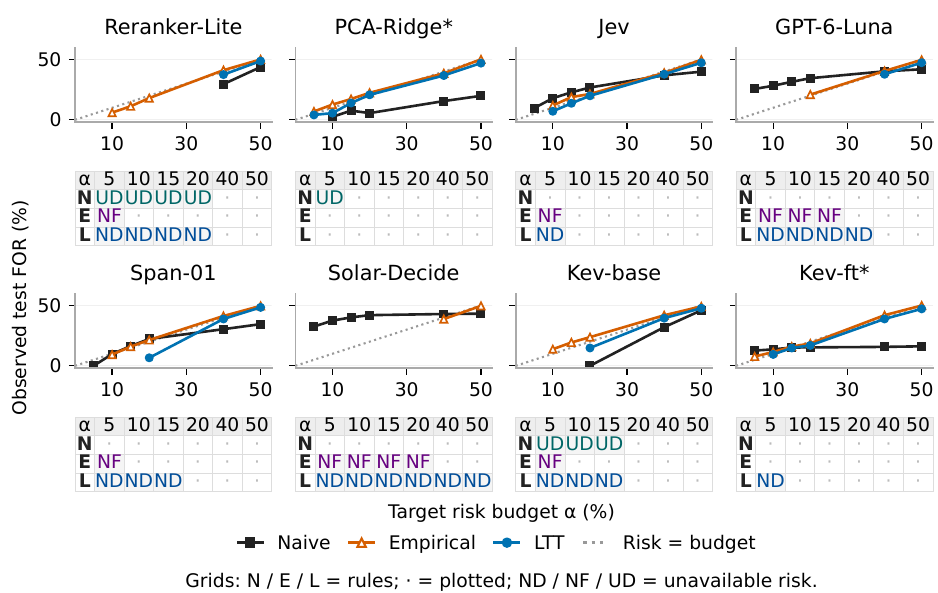}
\caption{SciDocs: test FOR versus risk budget.}
\end{figure*}
\begin{figure*}[t]
\centering
\includegraphics[width=\textwidth]{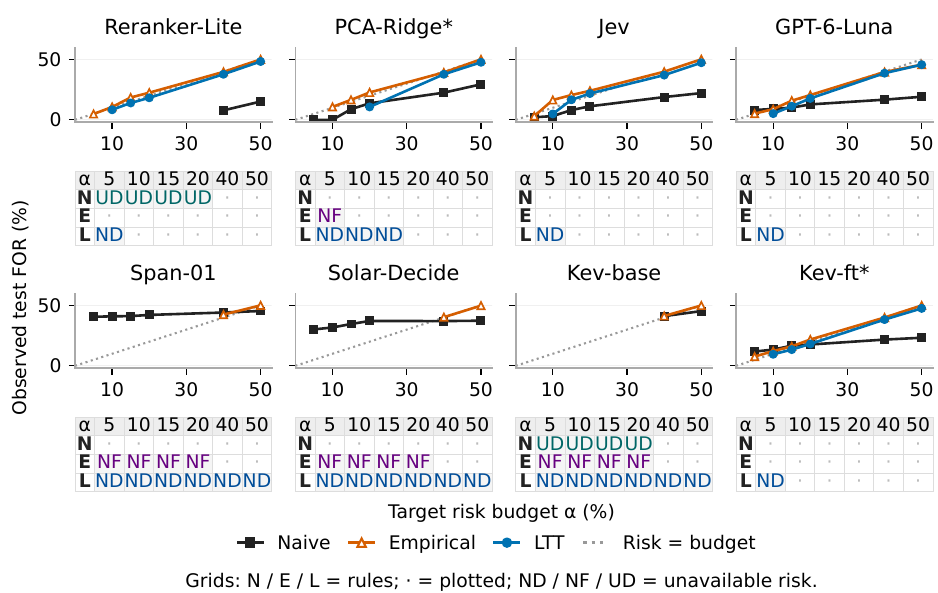}
\caption{StackOverflow: test FOR versus risk budget.}
\end{figure*}
\clearpage
\subsection{LTT Risk--Utility Across Datasets}\label{app:operating-points}
Figure~\ref{fig:operating-grid} shows LTT-selected operating points for all six datasets. Utility is benign acceptance for moderation and safety, and irrelevant removal for reranking. Missing certificates are never plotted as zero-risk points.
\begin{figure*}[ht]
\centering
\includegraphics[width=\textwidth]{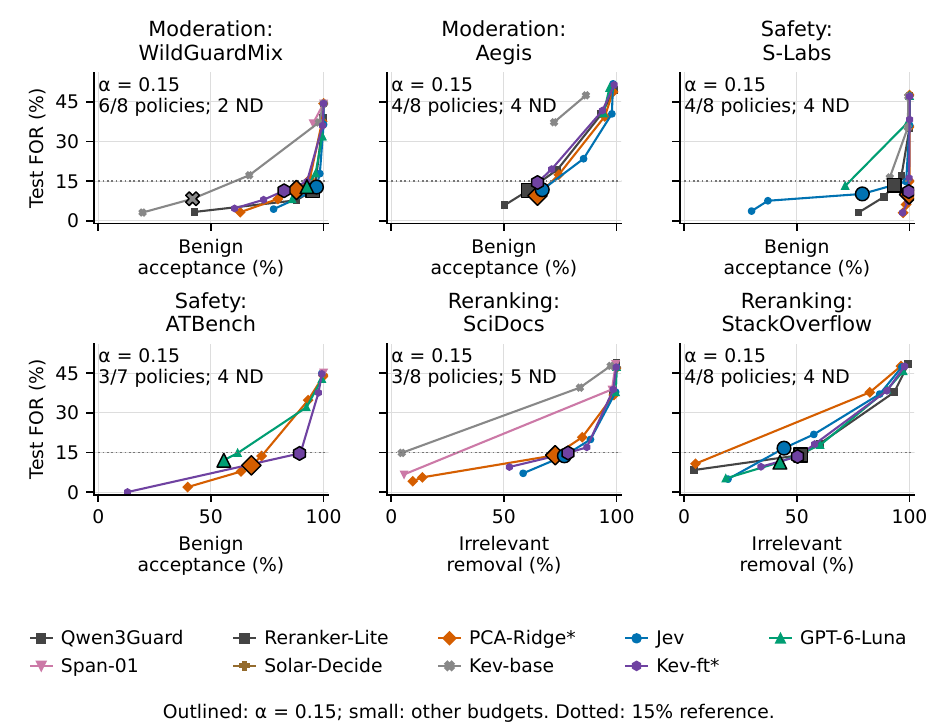}
\caption{LTT-selected operating points for all datasets. Large outlined markers identify $\alpha=0.15$; small markers show the other budgets. Counts indicate defined policies and ND outcomes at $0.15$. The dotted line is the 15\% reference. *: supervised adaptation.}
\label{fig:operating-grid}
\end{figure*}

\end{document}